\documentclass[journal]{IEEEtran}
\ifCLASSOPTIONcompsoc
  \usepackage[nocompress]{cite}
\else
  \usepackage{cite}
\fi
\ifCLASSINFOpdf
\else
\fi

\usepackage[cmtip,all]{xy}
\usepackage[english]{babel}
\usepackage{amsmath,amssymb}
\usepackage{times}
\usepackage{relsize}
\usepackage{graphicx}
\usepackage{url}
\usepackage{nicefrac}
\usepackage{booktabs}
\usepackage{multirow}
\usepackage{mparhack}
\usepackage{caption}
\usepackage{amsthm}
\usepackage{bm}
\usepackage{algorithm}
\usepackage[noend]{algorithmic}
\usepackage{tabularx}

\usepackage{graphicx}
\usepackage[caption=false,font=footnotesize]{subfig}

\usepackage{array}
\usepackage{eqparbox}
\usepackage{mdwmath}

\usepackage{epsfig}
\usepackage{epstopdf}
\usepackage{xcolor}
\usepackage{subfig}

\usepackage{color}
\usepackage{soul}



\newtheorem{theorem}{Theorem}
\newtheorem{lemma}[theorem]{Lemma}

\newtheorem{corollary}[theorem]{Corollary}

\newcounter{exam}
\newcounter{rema}
\newtheorem{remark}[rema]{Remark}

\DeclareFontFamily{OT1}{pzc}{}
\DeclareFontShape{OT1}{pzc}{m}{it}%
              {<-> s * [1.1] pzcmi7t}{}
\DeclareMathAlphabet{\mathpzc}{OT1}{pzc}%
                                 {m}{it}

\def\CE{{\mathcal E}}

\def\CI{{\mathcal I}}

\def\CC{{\mathcal C}}

\def\CH{{\mathcal H}}

\def\CA{{\mathcal A}}
\def\CK{{\mathcal K}}

\def\CT{{\mathcal T}}

\def\CA{{\mathcal A}}
\def\CG{{\mathcal G}}

\newcommand{\CV}{\mathpzc{V}}
\usepackage{accents}

\makeatletter
\newcommand{\raisemath}[1]{\mathpalette{\raisem@th{#1}}}
\newcommand{\raisem@th}[3]{\raisebox{#1}{$#2#3$}}
\makeatother

\begin{document}


\title{Optimal Routing for Federated Learning over Dynamic Satellite Networks: Tractable or Not?}
\author{Yi Zhao,
        Di Yuan,~\IEEEmembership{Senior Member, IEEE},
        Tao Deng,
        Suzhi Cao,
        and Ying Dong,~\IEEEmembership{Student Member, IEEE}
        
\thanks{Yi Zhao and Di Yuan are with the Department of Information Technology, Uppsala University, 75105 Uppsala, Sweden (e-mail: yi.zhao@it.uu.se; di.yuan@it.uu.se). Tao Deng is with the School of Computer Science and Technology, Soochow University, Suzhou 215006, China (e-mail: dengtao@suda.edu.cn). Suzhi Cao is with the Technology and Engineering Center for Space Utilization, Chinese Academy of Sciences, Beijing, China (e-mail: caosuzhi@csu.ac.cn). Ying Dong is with the School of Communications and Information Engineering, Chongqing University of Posts and Telecommunications, Chongqing, China (e-mail: yingd.cqupt@qq.com).}
\thanks{This work was supported by the Swedish Research Council with grant number 2022-04123.}
}

\maketitle

\pagestyle{empty}
\thispagestyle{empty}

\begin{abstract}
Federated learning (FL) is a key paradigm for distributed model learning across decentralized data sources. Communication in each FL round typically consists of two phases: (i) distributing the global model from a server to clients, and (ii) collecting updated local models from clients to the server for aggregation. This paper focuses on a type of FL where communication between a client and the server is relay-based over dynamic networks, making routing optimization essential. A typical scenario is in-orbit FL, where satellites act as clients and communicate with a server (which can be a satellite, ground station, or aerial platform) via multi-hop inter-satellite links. This paper presents a comprehensive tractability analysis of routing optimization for in-orbit FL under different settings. For global model distribution, these include the number of models, the objective function, and routing schemes (unicast versus multicast, and splittable versus unsplittable flow). For local model collection, the settings consider the number of models, client selection, and flow splittability. For each case, we rigorously prove whether the global optimum is obtainable in polynomial time or the problem is NP-hard. Together, our analysis draws clear boundaries between tractable and intractable regimes for a broad spectrum of routing problems for in-orbit FL. For tractable cases, the derived efficient algorithms are directly applicable in practice. For intractable cases, we provide fundamental insights into their inherent complexity. These contributions fill a critical yet unexplored research gap, laying a foundation for principled routing design, evaluation, and deployment in satellite-based FL or similar distributed learning systems.

\end{abstract}


\section{Introduction}
\label{sec:introduction}

Federated learning (FL) has emerged as a powerful paradigm for
distributed model training across multiple devices or data sources
without sharing raw data.  In FL, typically, clients perform local updates using
their respective data and send model updates to a central
server for parameter aggregation~\cite{pmlr-v54-mcmahan17a,ZHANG2021106775,9261995}. In a FL round, communications consist in two phases. In the distribution phase,
the global model is sent from the server node to clients to perform
local training.  Next, in the aggregation phase, the clients upload
their updated local models to the server for aggregation. The processes of distributing the global model to the clients and sending updates to the server are also referred to as \emph{downloading} and \emph{uploading}, respectively. If the communication resource is constrained, FL may involve client selection (CS) \cite{CS-1}, i.e., choosing a subset of clients to collect their updates, in the aggregation phase.\par
Recent studies have shown growing interest in integrating FL into satellite communication and computing systems (i.e., in-orbit FL), especially low-Earth-orbit (LEO) constellations \cite{LEO-1}. In this paradigm, LEO satellites act as clients, while a satellite, a ground station, or an aerial platform can typically be the server. In-orbit FL is driven by the need to alleviate satellite-ground bandwidth bottleneck, to support services with up-to-date models, and to protect data privacy. In-orbit FL empowers various applications such as mobile communication, Earth observation, disaster monitoring, and target tracking \cite{Track}.\par

\renewcommand{\arraystretch}{1.2}
\begin{table}[!t]
\centering
\caption{Comparison of terrestrial FL and in-orbit FL in their characteristics and resulting routing requirements.}
\label{tab:comp}
\begin{tabular}{l|cc}
\toprule
\textbf{} & \textbf{Terrestrial FL \cite{mobility}} & \textbf{In-orbit FL} \\ \midrule
\textbf{\begin{tabular}[c]{@{}l@{}}Network topology\end{tabular}} & Star & Mesh \\ \hline 
\textbf{\begin{tabular}[c]{@{}l@{}} Link connectivity \end{tabular}} & Continuous & Intermittent \\ \hline 
\textbf{\begin{tabular}[c]{@{}l@{}}Server-client connection\end{tabular}} & Direct & \begin{tabular}[c]{@{}c@{}}Relay-based \\ (via other clients)\end{tabular} \\ \hline 
\textbf{Mobility} & \multicolumn{1}{c}{\begin{tabular}[c]{@{}c@{}}   Client location\\     Channel quality    \\     \quad \end{tabular}} & \multicolumn{1}{c}{\begin{tabular}[c]{@{}c@{}}   Client location\\     Channel quality    \\     Network topology\end{tabular}} \\ \hline 
\textbf{\begin{tabular}[c]{@{}l@{}}Multi-hop routing\end{tabular}} & No & Yes  \\ \bottomrule
\end{tabular}
\end{table}

However, in-orbit FL introduces new challenges to routing among the server and clients, due to characteristics that fundamentally differ from terrestrial FL, as summarized in Table \ref{tab:comp}. In short, the core difference lies in relay-based communication over time-varying network topology. To cope with this challenge, some studies \cite{liu2025fedhchierarchicalclusteredfederated, 10605604,10121575} have utilized clustered FL with hierarchical servers. The satellites form multiple clusters, and each cluster has a local server (which is, one of the satellites) for intra-cluster aggregation. The connections between the local server and the clients in a cluster can be direct based on distance \cite{liu2025fedhchierarchicalclusteredfederated}, or relay-based but over stable links within one orbit \cite{10605604, 10121575}. The communication among the local server of each cluster and the super server is either simply assumed to be direct \cite{10121575} or handled by the underlying black-box network-layer protocols \cite{10605604}, or via topology-aware multi-hop routing algorithms based on orbital predictability and network flow theory \cite{liu2025fedhchierarchicalclusteredfederated}. Papers \cite{10716798, 10920666} have concerned aggregation scheduling and communication efficiency for in-orbit FL, but do not consider routing. Overall, the routing design constitutes a critical gap in this research landscape, as routing is not an optional add-on but a fundamental problem under relay-based communications. Besides, the performance of routing solution directly impacts the feasibility or determination of contributors (i.e., CS) to global aggregation and hence influences the convergence efficiency of FL \cite{10716798}. Therefore, tailoring dedicated routing schemes for in-orbit FL, instead of resorting to the off-the-shelf application-agnostic network-layer protocols, is necessary, in particular given the constrained communication resource in satellite networks. Although the most recent work \cite{liu2025fedhchierarchicalclusteredfederated} represents a significant step in addressing this research gap, it exhibits the limitations of myopic snapshot-by-snapshot greedy-based routing optimization and the communication restricted to unicast and splittable flow (which means an ideal assumption that a data packet can be arbitrarily partitioned). Practical factors such as CS, cross-layer integration with routing protocols, and concurrent FL of multiple models have not been involved in \cite{liu2025fedhchierarchicalclusteredfederated}. 

Segment routing (SR) \cite{SR-1} and time-varying graph (TVG) bring new possibilities for routing in in-orbit FL. SR enables flexible path control at a source node, without maintaining per-flow state in intermediate nodes, supporting also multicast and multipath transmission. SR is especially efficient for traffic engineering in highly dynamic yet predictable satellite networks \cite{SSR-1, SSR-2}. To model the dynamics of satellite constellations and the cache-and-forward capability of satellites, a TVG consists of a sequence of connected static topologies (a static topology corresponds to a so-called snapshot) along time. Such models are well-studied in the networking and
distributed systems literature, e.g.,~\cite{KEMPE2002820,casteigts2012}.  
For satellite routing, the periodicity and predictability of orbits
make time-varying graphs particularly well suited.

Based on SR and TVG, given the need to route FL model parameters over time-varying
network topologies with a deadline, the potential presence of multiple FL models and
multiple aggregation nodes, the possibility of employing multicast
rather than unicast for global model distribution, and possibly the
need of making CS for model collection, we ask a central
and fundamental question: \par
\emph{\textbf{Question:} Is optimal routing of FL model(s) over TVGs for satellite networks computationally tractable under these diverse
scenarios? }\par
More concretely, can we design a polynomial-time algorithm
that guarantees the global optimum, or for certain routing settings,
the problem is inherently NP-hard and we probably need to resort to
sub-optimal routing solutions? Analyzing the underlying computational complexity is crucial: Polynomial-time algorithms are valuable for practical use in-orbit FL, while rigorous understanding of the problems that fall into the class of NP-hardness delineates the boundaries of what is theoretically achievable, guiding the design of algorithms and system architectures\footnote{Although, in some realistic systems, especially those of small scale, or where a routing scheme exhibits desirable properties such as capacity-preserving, heuristic solutions may still perform satisfactorily for a NP-hard problem.}. 

To our knowledge, this is the first work that systematically
investigates the tractability of optimal routing for FL over
dynamic satellite networks.  In the forthcoming sections, we examine a
broad spectrum of problem settings and provide a comprehensive
tractability analysis, with the following specific contributions.

\begin{itemize}

\item We formally model optimal routing for global model distribution across
time-varying satellite networks under variants defined by:

\begin{itemize}
\item whether routing is unicast or multicast,
\item whether there is a single FL model or multiple heterogenous models,
\item whether the data flow is splittable (i.e., multiple
paths can be used to reach a destination, each carrying a 
fractional amount of flow) or not,
\item and different objective functions such as average (sum) completion time or min-max 
completion time, measured over snapshot indices.
\end{itemize}
For each of the problem settings above, we rigorously prove either that the global
optimum can be achieved in polynomial time or that the problem is
NP-hard, thus delineating precise tractability boundaries.

\item For model aggregation (collection of local updates), we show that the
 boundary between tractability and intractability hinges critically on:
\begin{itemize}
\item whether there is a single FL model or multiple heterogenous models,
\item whether the data flow is splittable or not,
\item whether there is CS or not.
\end{itemize} 

\end{itemize}

By formally mapping a complexity landscape for FL routing over satellite networks, we lay a solid foundation for principled design, evaluation, and
deployment of routing strategies in satellite-based and other dynamic distributed learning systems, filling a critical yet unexplored research gap.

\begin{remark}
In our analysis, we consider optimal routing of one and two FL models.
Note that the result of intractability of a problem with two models
immediately generalize to scenarios with more than two models.  If a
problem is tractable for two models, we will show that tractability
holds also for multiple models.  
$\Box$
\end{remark}
\begin{remark}
We assume satellite(s) as server(s) in the analysis in the rest of the paper. Nevertheless, the results apply also to the case of ground station or aerial platform.$\Box$	
\end{remark}

\section{Preliminaries}
\label{sec:preliminary}

\subsection{In-Orbit FL with Deadline, CS, and Multiple Models}
\begin{figure*}[t]
\centering
\subfloat[FL of one model, without CS.]{
    \includegraphics[width=0.48\textwidth]{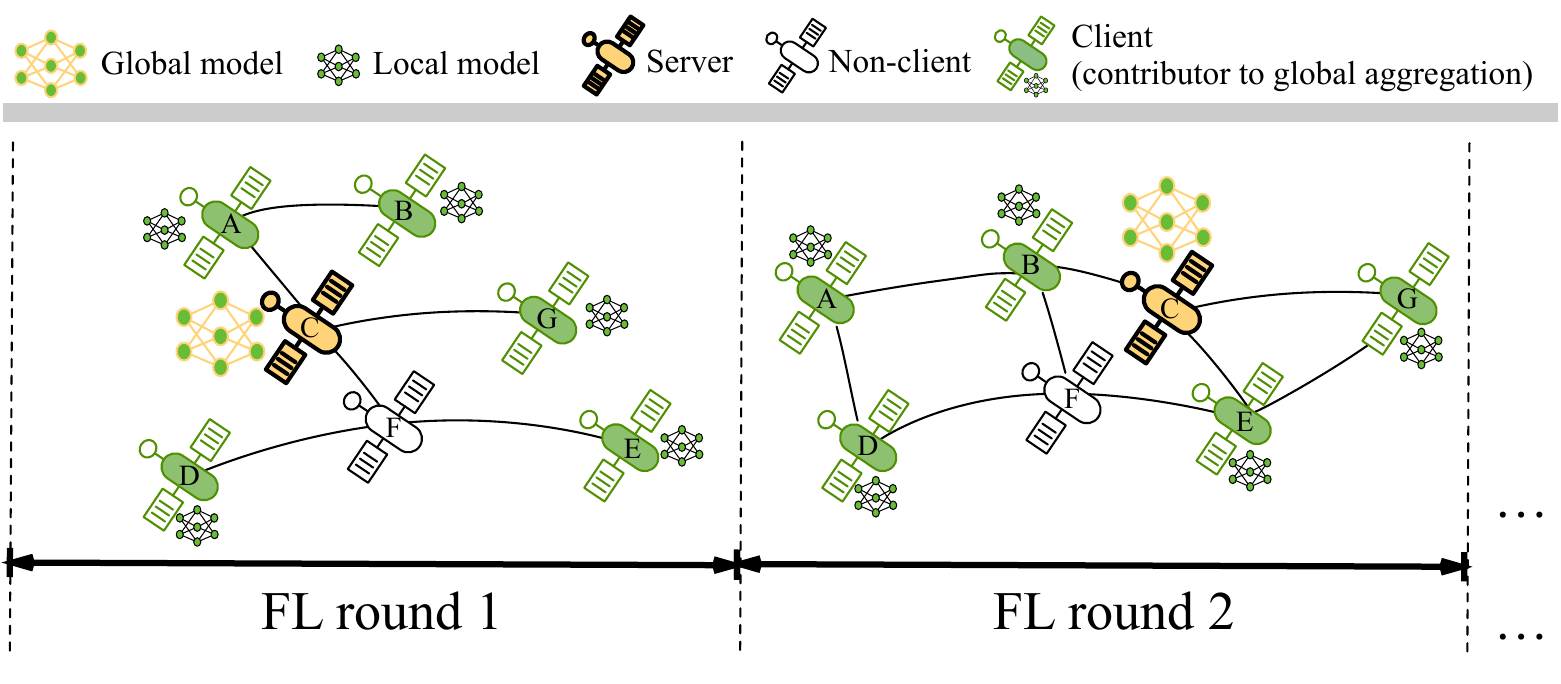}
    \label{fig:sub1}}\hfill
\subfloat[FL of one model, with CS.]{
    \includegraphics[width=0.48\textwidth]{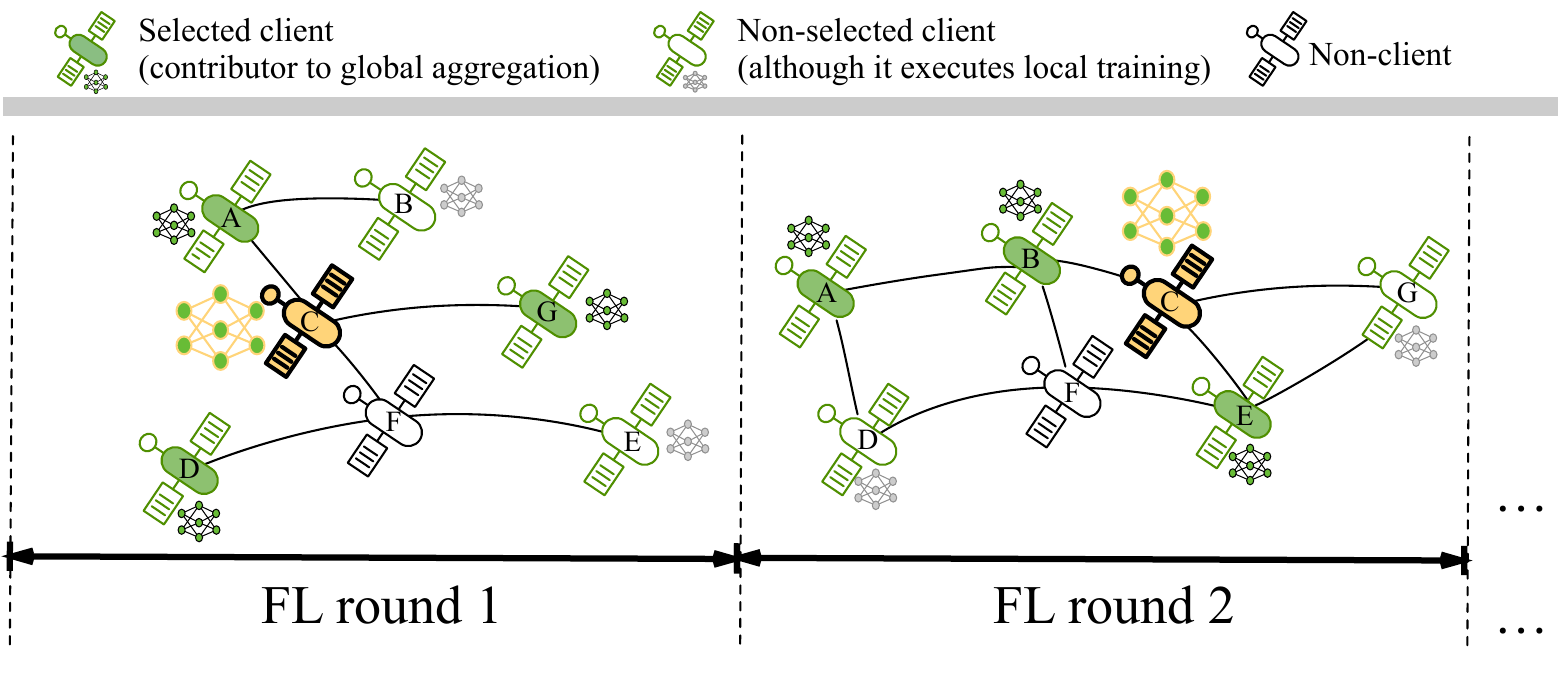}
    \label{fig:sub3}}\hfill
\vspace{0.3cm}
\subfloat[FL of two models, without CS.]{
    \includegraphics[width=0.48\textwidth]{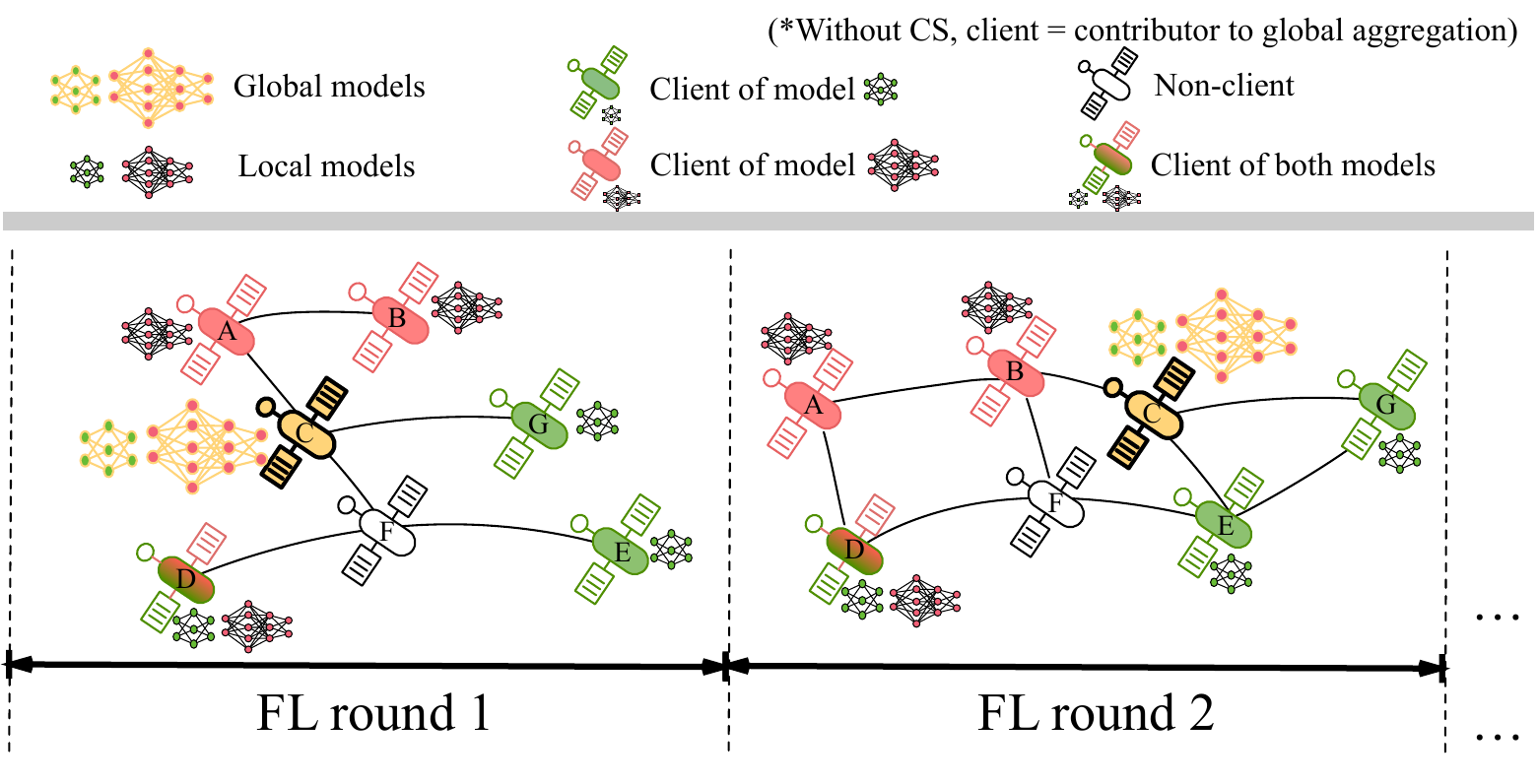}
    \label{fig:sub2}}
\subfloat[FL of two models, with CS.]{
    \includegraphics[width=0.48\textwidth]{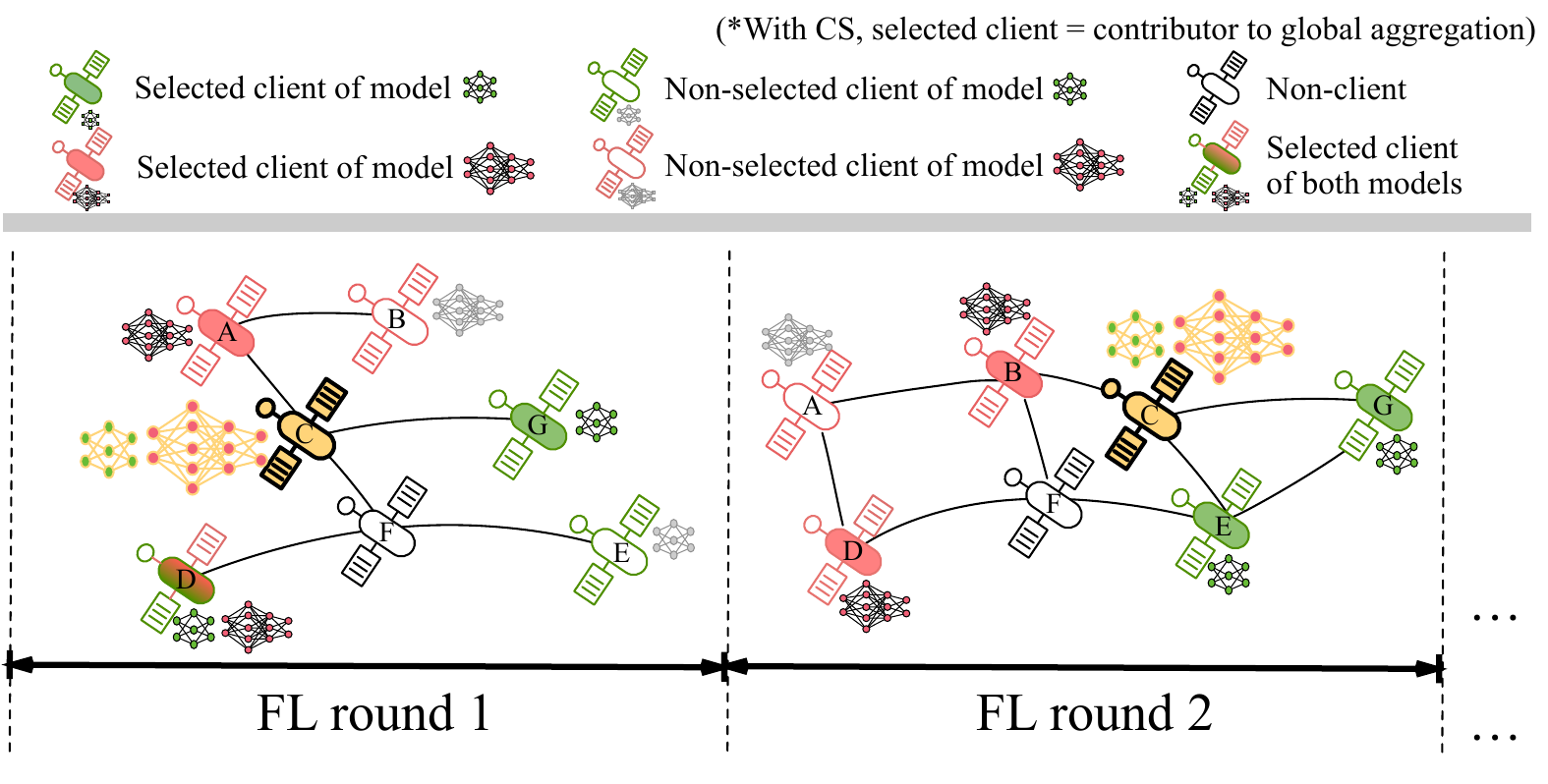}
    \label{fig:sub4}}
\caption{In-orbit FL, under different options of: 1) one or two model(s) and 2) with or without CS.}
\label{fig:CS and MFL}
\end{figure*}

We impose a deadline for each round of in-orbit FL, although the durations of different rounds need not be uniform. This is motivated by two reasons: first, the absence of a deadline may degrade FL convergence, in particular in communication-constrained satellite networks; second, this setting facilitates modeling the dynamics of satellite network topology (see Section \ref{sec:TVG} for details).\par
Sometimes, due to limited communication resources, it is not feasible to collect the local models from all clients before the deadline of an FL round. In this case, CS \cite{CS-1} becomes necessary, which allows a subset of selected clients to contribute to the global model aggregation. Since not all clients contribute equally to the global aggregation in accelerating convergence, CS is typically biased based on some utility metrics\cite{CS-1,CS-2}. The higher the utility value, the more important the local updates of a client are for the FL convergence. CS with static metrics (such as dataset size and hardware parameters) is typically conducted by the server before the model distribution phase, while CS relying on real-time metrics (such as channel information and local training loss) is performed in the aggregation phase \cite{CV-1,CNV-1,CS-1,CS-2}. In this paper, we focus on CS with real-time client utilities. For FL without CS, clients refer to the satellites involved in local training as well as global model aggregation. For FL with CS, the term client means a satellite executing local training, while selected clients are the satellites that eventually contribute to global model aggregation.\par
We go beyond a single FL model in our tractability analysis. Multiple FL models can be deployed for different in-orbit tasks (with essentially different input and output), or there could be heterogeneous models for the same task. The motivation of deploying the latter is usually for adaptive utilization of orbit-position-varying energy supply, efficient two-stage inference, and fault tolerance. For both cases, multiple FL models may have their respective servers or share one. We assume that the training of multiple FL models, no matter with one or more servers, share the same time horizon setting (the start time and deadline) of FL rounds. This setting facilitates the management of interactions between training processes of models. A typical example is knowledge distillation \cite{KD-1,KD-2} among heterogeneous models for the same task. The transmission of the models share the communication capacity of inter-satellite-links of the dynamic networks. \par 
An illustration showing in-orbit FL with options of one or two models, with or without CS, is given in Figure \ref{fig:CS and MFL}.

\subsection{Segment Routing, Multicast, and Multipath Transmission}
In a routing problem, the terms \emph{source} and \emph{sink} refer to the origin and destination of a data flow, respectively. And the term \emph{demand} refers to the amount of flow that a source/sink node requires to supply/receive. For the downloading phase of in-orbit FL, the source is the server while the sinks are the clients. Their roles reverse for the uploading phase. Additionally, if CS is applied, the sources in the uploading phase are instead the selected clients. \par
SR \cite{SR-1, SR-2} is a source routing paradigm. Unlike traditional routing where each intermediate node independently decides the next hop, SR enables flexible path control by allowing the source node to pre-push an ordered list of forwarding instructions, called a segment stack, into the packet header. In SR, an intermediate node simply forwards the packet to a specified node (with node segment) or over a specific interface/link (with adjacency segment) according to the segment instructions. SR enables simplified network operations, improved scalability, and precise traffic engineering. SR is typically implemented as an extension to the existing interior and border gateway protocols \cite{SR-1}. \par 
Except for unicast (from a single source to a single sink), SR natively supports multicast (from a single source to  multiple sinks), through segments including replication instructions to represent a multicast tree \cite{SSR-2}. At a branching node, the replication instruction indicates duplicating the packet and forwarding its copies along multiple downstream paths. For multipath transmission, the source can split the data into multiple packets and assign each packet a different segment stack (i.e., splittable flow), thereby utilizing path diversity for load balancing. With source node A, the examples of simplified representations of segment stacks for unicast and multicast are given below, illustrated in Figure \ref{fig:sr}:
\begin{itemize}
\item Unicast along path A-B-C-D: [B, C, D] 
\item Unicast with multipath transmission, via paths A-B-C-D and A-E-F-D: [B, C, D] for packet 1 and [E, F, D] for packet 2.
\item Multicast from A to C, E, and F, branching at B and D: [A, B, Replication, [C], [D, Replication, [E], [F]]]
\end{itemize}
\par
\begin{figure}
\centering
  \includegraphics[width=0.37\textwidth]{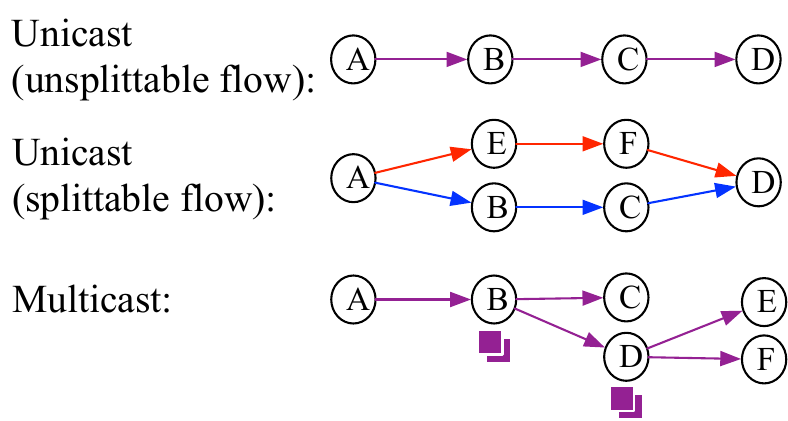}
  \caption{Examples of multipath transmission and multicast in SR, where overlapped purple blocks indicate replication operations.}
  \label{fig:sr}
\end{figure}
The application of SR in satellite networks is motivated by the need to cope with highly dynamic yet predictable topologies, to simplify onboard processing through stateless forwarding, and to achieve load balancing under constrained communication resources. Recent studies have demonstrated the effectiveness of SR in LEO networks for traffic engineering and fast reroute \cite{SSR-1, SSR-2}.\par
In our scenario, for the downloading phase of FL, the server determines the routing paths based on predictable global network topologies, and pre-pushes the segments in the header of data packets of the global model. In the uploading phase, the server determines CS (if any) as well as routing, and informs the clients. The (selected) clients will then pre-push the corresponding segments in their data packets of local updates.

\subsection{Time-Varying Graphs}
\label{sec:TVG}
Let $\CI = \{1, \dots, I\}$ and $\CK = \{1, \dots, K\}$ denote the
sets of satellites and snapshots, for a generic FL round under
consideration. If there is one FL model, the model size is denoted by
$q$, the server satellite is denoted by $s$, and the set of clients
for local training is denoted by $\CC$. For the scenario of two FL
models, $q_1$ and $q_2$ denote their respective sizes, $\CC_1$ and
$\CC_2$ denote the respective clients for local training, and, if the
two models have distinct server satellites, notations $s_1$ and $s_2$
are used.  We allow $\CC_1$ and $\CC_2$ to overlap in our tractability
analysis.\par

\begin{figure}
\centering
  \includegraphics[width=0.42\textwidth]{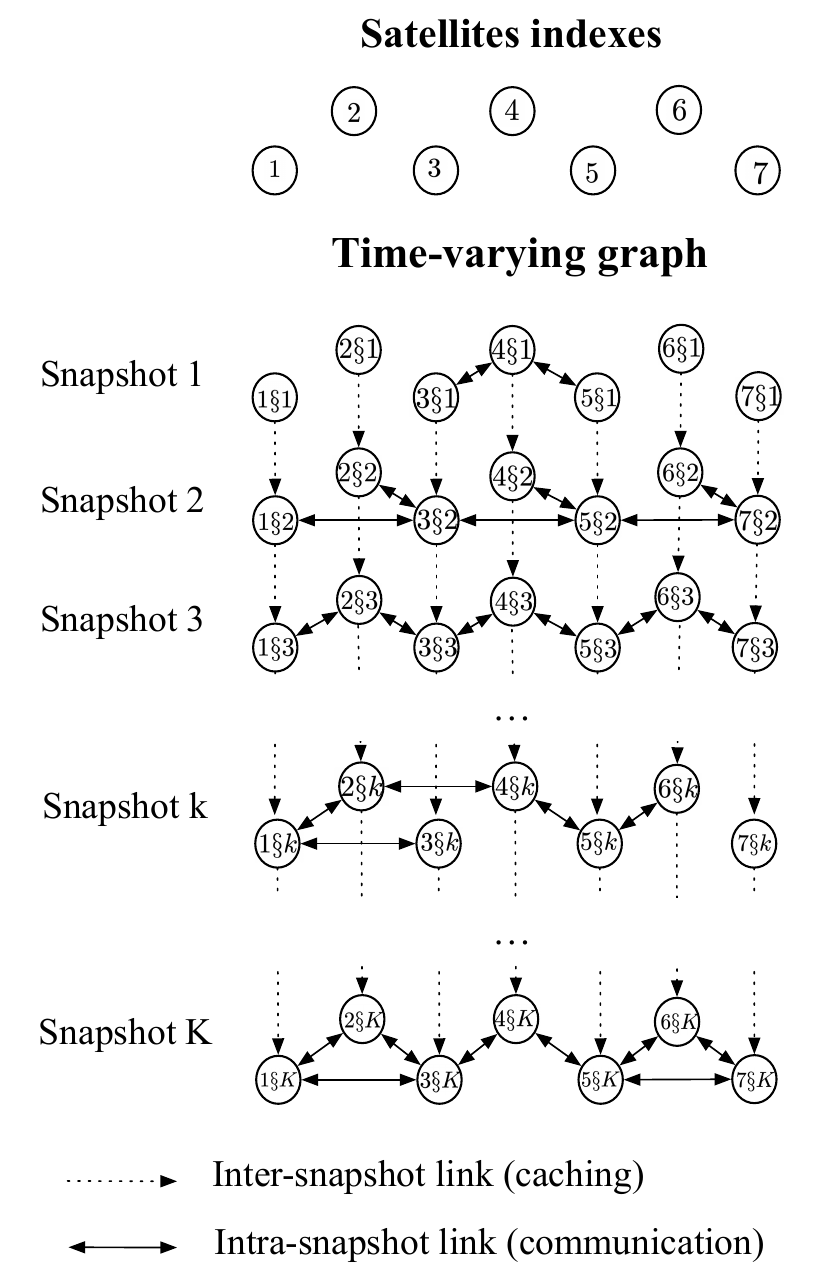}
  \caption{An example of TVG, where $i \S k$ represents the copy of satellite $i$ in snapshot $k$.}
  \label{fig:tvg}
\end{figure}

TVG is a concept for representing time-dynamic
graph structure, frequently used for modeling satellite
networks~\cite{TVG-1,TVG-3,s22124552,9931973}. Typically, a TVG is
composed by a sequence of connected snapshots, and each snapshot is a static
network topology, illustrated in
Figure~\ref{fig:tvg}. We use $\CG = (\CV, \CA)$ to denote the TVG with node set $\CV$ and arc set $\CA$. As a satellite appears as nodes in multiple snapshots in $\CG$, we need to notationally distinguish between these nodes, and at the same time indicate they all refer to the same satellite. To this end, $i \S k$ is used to refer to the node of satellite $i$ in snapshot $k$. \par
\begin{figure*}[t]
\centering
\subfloat[Unicast with unsplittable flow, for FL of one model.]{
    \includegraphics[width=0.21\textwidth]{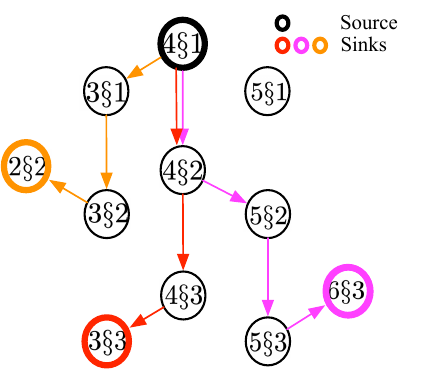}
    \label{fig:D-sub1}}\hfill
\subfloat[Unicast with splittable flow, for FL of one model.]{
    \includegraphics[width=0.21\textwidth]{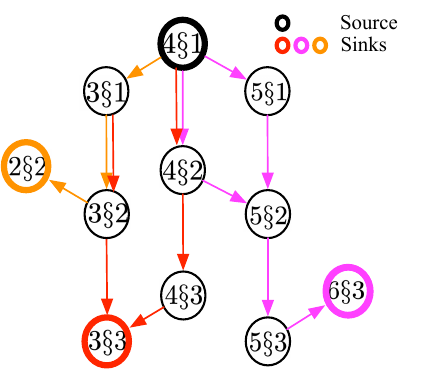}
    \label{fig:D-sub2}}\hfill
\subfloat[Multicast (only unsplittable flow) for FL of one model.]{
    \includegraphics[width=0.21\textwidth]{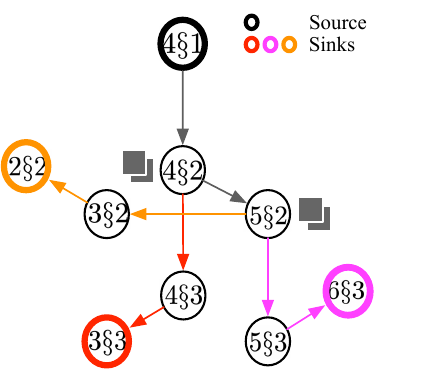}
    \label{fig:D-sub3}}\hfill
\subfloat[Multicast (only unsplittable flow) for FL of two heterogenous models.]{
    \includegraphics[width=0.24\textwidth]{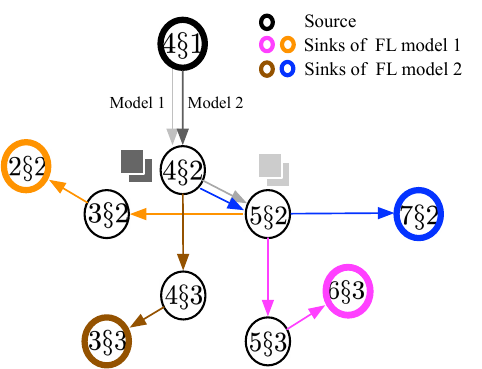}
    \label{fig:D-sub4}}
\caption{Examples of routing solutions for the downloading phase, under different options of: 1) unicast or multicast, 2) unsplittable and splittable flow, and 3) one or two FL models, where satellite 4 is the server and the overlapped gray blocks indicate the data packet replications.}
\label{fig:D-USS-UM-12}
\end{figure*}

\begin{figure}[t]
\centering
\subfloat[No CS.]{
    \includegraphics[width=0.23\textwidth]{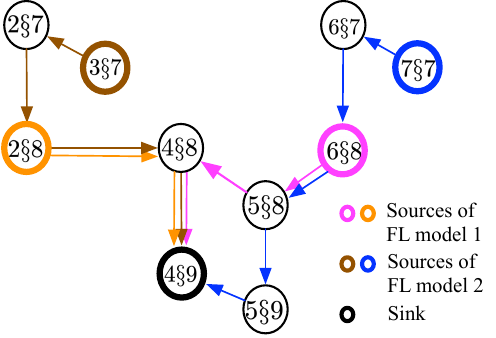}
    \label{fig:U-sub2}}\hfill
\subfloat[CS, with selected clients 2 and 7.]{
    \includegraphics[width=0.21\textwidth]{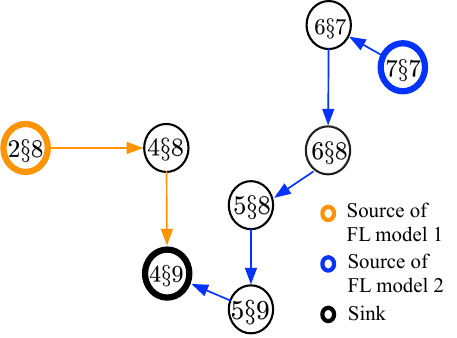}
    \label{fig:U-sub4}}
\caption{Examples of routing solutions for the uploading phase with unsplittable flow, with or without CS (ref. the clients and corresponding downloading routing in Figure  \ref{fig:D-sub4}).}
\label{fig:U-USS-CS-12}
\end{figure}

\begin{remark}
In the following, notation $i$ and $j$ are reserved for satellite indices, whereas $u$ and $v$ denote nodes in graph $\CG$. A node here is a satellite in a snapshot. The previously introduced concepts of source and sink nodes are thus extended from a satellite to its copy in a specific snapshot. Notational style of $i \S k$ is used for a node, with an exception of only one snapshot where we simply use $i$ instead of $i \S 1$. $\Box$
\end{remark}

There are two types of links in a TVG: Intra-snapshot links and inter-snapshot links. An intra-snapshot link, i.e., a bi-directional link between two satellites in a snapshot (see Figure~\ref{fig:tvg}), means a direct and stable communication connection between them within the duration of this snapshot. Unless otherwise stated, arc $(u,v)$ represents a directional arc from node $u$ to node $v$. In either direction of arcs $(u,v)$ and $(v,u)$, the capacity, denoted by $p_{uv}$ and $p_{vu}$ respectively, is the amount of data that can be accommodated within the snapshot duration, i.e., the product of average data transmission rate and the snapshot duration. Sometimes in the analysis, we normalize the capacity with respect to FL model size. An inter-snapshot link is present between the two nodes of a satellite in two consecutive snapshots, indicated by a vertical dashed arc in Figure~\ref{fig:tvg}. This type of link is used for modeling caching of data on the satellite from one snapshot to the next. More specifically, data arriving at a satellite in a snapshot become available on this satellite, and can be further transmitted from this satellite to others via intra-snapshot links in the subsequent snapshots. For an inter-snapshot link $(u,v)$ across snapshots $k$ and $k+1$ of satellite $i$, i.e., $u = i \S k$ and $v = i \S (k+1)$, we assume that the cache capacity is sufficiently large and does not impose constraints on routing, and hence its capacity $p_{uv}$ can be set to $q |\CC|$ and $q_1 |\CC_1| + q_2 |\CC_2|$ for the cases of one and two models, respectively. \par
 We illustrate different routing options for downloading phase using TVGs in Figure \ref{fig:D-USS-UM-12} (see Section \ref{sec:multicast} for the reasons why splittable flow is not considered for multicast). The same concepts of single/multiple FL models and splittable/unsplittable flow apply to uploading phase. In addition, the CS option for uploading phase is shown in Figure \ref{fig:U-USS-CS-12}.

\subsection{Objective Function}
\label{sec:obj}
\begin{figure}
\centering
  \includegraphics[width=0.48\textwidth]{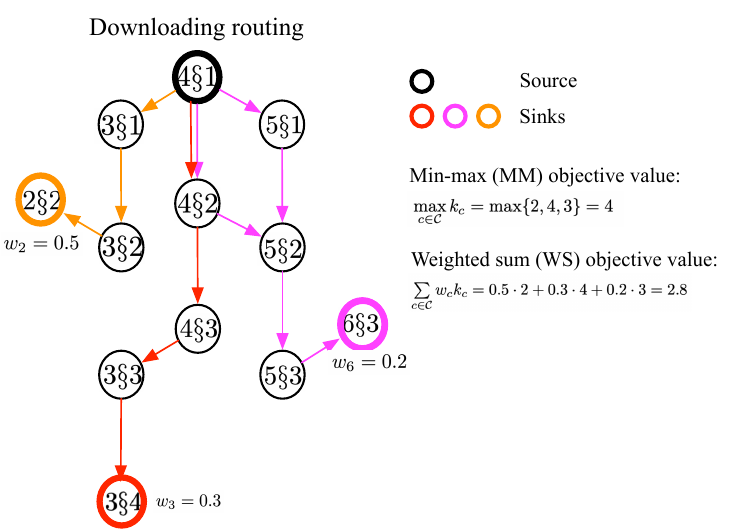}
  \caption{The two types of objective functions and their values based on a routing decision for downloading with one FL model, unicast, and splittable flow.}
  \label{fig:obj}
\end{figure}

For global model distribution, i.e., the downloading phase, the target of routing optimization is to minimize the completion time for early start of local training. This idea leads to two specific objective functions to minimize: (i) the weighted sum (WS) of the model arrival snapshot indices over the clients, and (ii) the maximum snapshot index (min-max, MM) by which all clients have received their global model(s). Denote by $k_c$ the snapshot by which client $c$ receives its global model(s). Denote by $w_c$ the weight of client $c$, reflecting its priority based on, e.g., the time needed for local training. The WS objective function is then $\sum_{c \in \CC}w_c k_c$ for single model and $\sum_{c \in \CC_1 \cup \CC_2}w_c k_c$ for two model(s). The MM objective function is $\max_{c \in \mathcal{C}}{k_c}$ for single model and $\max_{c \in \mathcal{C}_1 \cup \mathcal{C}_2}{k_c}$ for two models. While the WS objective function captures the needs of individual clients, the MM objective function is preferable when fairness is prioritized. A numerical example of the two types of objective functions is shown in Figure \ref{fig:obj}.

For the uploading phase, the objective is to receive the local models from (selected, if with CS) clients before the deadline of the current round. We address scenarios without and with CS. Without CS, the uploading routing problem is a feasibility problem to determine whether or not routing of the models of the clients can be accomplished by the last snapshot $K$\footnote{The averaging operation and global parameter updates on the server are independent of network topology and occur at the end of each FL round, hence we assume sufficient time is reserved after snapshot $K$ for this purpose.}. CS may become necessary due to limited communication resources. With CS, the routing problem amounts to utility maximization subject to the constraints that all selected models must be delivered by snapshot $K$. Here $u_c$ denotes the utility of client $c$, representing the importance of its local updates for the global aggregation. In addition, let $\kappa_c$ denote the earliest snapshot from which routing for client $c$ can begin (clients may finish local training and thus start uploading in different snapshots).\par
\begin{remark}
In this paper, we use the snapshot index as the time metric, that is, we aim to complete the routing within as early a snapshot as possible. Note that considering snapshot index for the objective function does not cause loss of generality, since a snapshot can be further sliced if better granularity is desired.$\Box$
\end{remark}

\subsection{Problem Tractability}
\label{sec:3sat}

In the realm of combinatorial optimization, assessing tractability is
fundamental to understanding whether a problem can be solved
efficiently and at scale. Tractability here refers to the existence of
a polynomial-time algorithm that guarantees global optimum. When such
an algorithm is absent (or not likely to be found), the problem is considered intractable,
typically belonging to the class of NP-complete or NP-hard problems. While the distinction is well known, it is worth recalling that
NP-completeness refers specifically to decision problems within the
class NP, whereas NP-hardness encompasses a broader class of problems
that are at least as difficult, without the requirement that they are
in NP. For an optimization problem, if its decision version is NP-complete,
then the optimization problem itself is NP-hard~\cite{du2022introduction}. Note that an NP-hardness classification is a worst-case
characterization. Approximation algorithms, relaxations, and
heuristics may still yield high-quality solutions for large instances
of an NP-hard problem. Nevertheless, a tractability analysis
provides a theoretical characterization of computational limits, to
form a conceptual foundation for the algorithmic design choices for
complex systems, such as distributed learning.\par

Establishing tractability or intractability generally relies on the
following concepts.  Polynomial-time solvability can be proven, either
by explicitly constructing an algorithm that always attains the global
optimum, or by demonstrating that the problem lies in a sub-class of
some known tractable problem. NP-hardness is typically proven via
polynomial-time reduction from a known NP-hard problem (such as one
in the list provided in~\cite{GareyJohnson1979}), i.e., the known hard
problem is a sub-class of our target problem. We remark that for the
hardness proofs later in the paper, the reductions themselves are
clearly polynomial time and therefore we will not state this fact
explicitly in the proofs.

\textbf{[3SAT Problem]} Since some of our analysis of problem reduction will use the well-know NP-complete problem
3-Satisfiability (3SAT), we give the problem definition here.
Consider $m$ Boolean variables and $n$ clauses.  For every variable
$x_\ell, \ell = 1,\dots, m$, there are two literals: the variable itself
$x_\ell$ and its negation ${\bar x}_\ell$, giving a total of $2m$
literals. A clause $e_h,  h = 1, \dots, n$ is a disjunction of exactly
three literals.  The 3SAT formula is the conjunction of all clauses,
and the question is whether or not there is an assignment of
truth/false values to the $m$ variables such that all $n$ clauses hold
true. Throughout the paper, we assume without any loss of generality
that a given 3SAT instance is irreducible. For example, the two
literals of a variable both appear in some (but not the same) 
clauses as otherwise setting its value is trivial.


\section{Global Model Distribution: Unicast}
\label{sec:unicast}

\begin{table}[!t]
\centering
\caption{Definitions of the used notations.}
\label{tab:notation}	
\begin{tabular}{lc}
\toprule
\textbf{Notations} & \textbf{Definitions} \\ \midrule
$\CI$ & Set of satellites \\
$\CK$ & Set of snapshots \\
$I$ & Number of satellites \\
$K$ & Number of snapshots \\ 
$\CC$ & Set of clients (Client sets $\CC_1$ and $\CC_2$ for two models)\\
$s$ & Server (Servers $s_1$ and $s_2$ for two models)\\
$q$ & Model size (Model sizes $q_1$ and $q_2$ for two models)\\
$\CG= (\CV, \CA)$ & \begin{tabular}[c]{@{}c@{}}TVG with node set $\CV$ and arc set $\CA$\end{tabular} \\
$i \S k$ & Satellite $i$ in snapshot $k$ \\
$(u,v)$ & \multicolumn{1}{c}{\begin{tabular}[c]{@{}l@{}}Arc from node $u$ to node $v$ in a TVG\end{tabular}} \\
$p_{uv}$ & \multicolumn{1}{c}{Capacity of arc $(u,v)$}\\
$w_c$ & Weight of client $c$ \\
$k_c$ & \multicolumn{1}{c}{Snapshot by which client $c$ finishes the downloading}\\
$\kappa_c$ & Snapshot from which client $c$ starts its uploading\\
$u_c$ & Utility of client $c$ for CS\\
$x_\ell$ and $\bar x_\ell$ & The $\ell$th variable and its negation in a 3SAT problem\\ 
$e_h$ & The $h$th clause in a 3SAT problem\\
$a_c$ & Auxiliary node for client $c$\\
$\CT^M$ & Set of terminal nodes in a Steiner arborescence\\
$\CH^*$ & Optimal minimum-cost Steiner arborescence\\
\bottomrule
\end{tabular}
\end{table}

We start by considering global FL model distribution with unicast
routing, where each client has its own data flow from its server.
Here, for each client, one can opt to restricting the entire model to
be routed on a single path, i.e., unsplittable flow, or allowing the
use of multiple paths, each carrying a fractional amount of flow,
named splittable flow (aka bifurcated flow). Whereas unsplittable flow
avoids the complexity of reassembling packets at the destination and
simplifies routing configuration, splittable flow is generally able to
utilize the network capacity better with load balancing. Therefore we consider both cases
in our analysis. Hence, there are eight problem variants by the choice of objective
function, whether the flow is splittable, and whether there is one or
two FL models. For the scenario of two FL models, we assume that they
share a common server, and discuss later the case of separate servers. For clarity, each problem variant is denoted by a three-part
abbreviation in the form [Number of models]-[flow type]-[objective].
Specifically, we use 1 or 2 to indicate single or dual FL models, SF
and UF for splittable or unsplittable flows, and WS or MM for
weighted-sum or min-max objectives, respectively. The summary of notations used in this paper are given in Table \ref{tab:notation}. The tractability results under unicast for the downloading phase are shown in Table~\ref{tab:downloading}.\par

\begin{table}[t]
\centering
\caption{Problem variants and their computational complexity for unicast global model distribution.
(UF = unsplittable flow, SF = splittable flow, WS = weighted sum, MM =
min-max, P = polynomial time.)}
\begin{tabular}{c|cc|cc}
\toprule
\multirow{2}{*}{\textbf{Models}} &
\multicolumn{2}{c|}{\textbf{UF}} &
\multicolumn{2}{c}{\textbf{SF}} \\
\cmidrule(lr){2-3} \cmidrule(lr){4-5}
 & \textbf{WS} & \textbf{MM} & \textbf{WS} & \textbf{MM} \\
\midrule
\textbf{1} &
\begin{tabular}[c]{@{}c@{}}1-UF-WS\\ \textit{(P)}\end{tabular} &
\begin{tabular}[c]{@{}c@{}}1-UF-MM\\ \textit{(P)}\end{tabular} &
\begin{tabular}[c]{@{}c@{}}1-SF-WS\\ \textit{(NP-hard)}\end{tabular} &
\begin{tabular}[c]{@{}c@{}}1-SF-MM\\ \textit{(P)}\end{tabular} \\[4pt]
\textbf{2} &
\begin{tabular}[c]{@{}c@{}}2-UF-WS\\ \textit{(NP-hard)}\end{tabular} &
\begin{tabular}[c]{@{}c@{}}2-UF-MM\\ \textit{(NP-hard)}\end{tabular} &
\begin{tabular}[c]{@{}c@{}}2-SF-WS\\ \textit{(NP-hard)}\end{tabular} &
\begin{tabular}[c]{@{}c@{}}2-SF-MM\\ \textit{(P)}\end{tabular} \\
\bottomrule
\end{tabular}
\label{tab:downloading}
\end{table}

\subsection{Tractable Problem Variants}

\begin{theorem}
\label{theo:1ufws}
1-UF-WS is solvable in polynomial time.
\end{theorem}
\begin{proof}
To prepare for the proof, we introduce an expanded topology for 
the inter-snapshot links of each client, as illustrated in
Figure~\ref{fig:mcf}. For each client $c$, an auxiliary node
$a_c$ is added. An arc is added to connect each node, $c \S k$, $k=1, \dots, K$, to node $a_c$. 

\begin{figure}
\centering
  \includegraphics[width=0.35\textwidth]{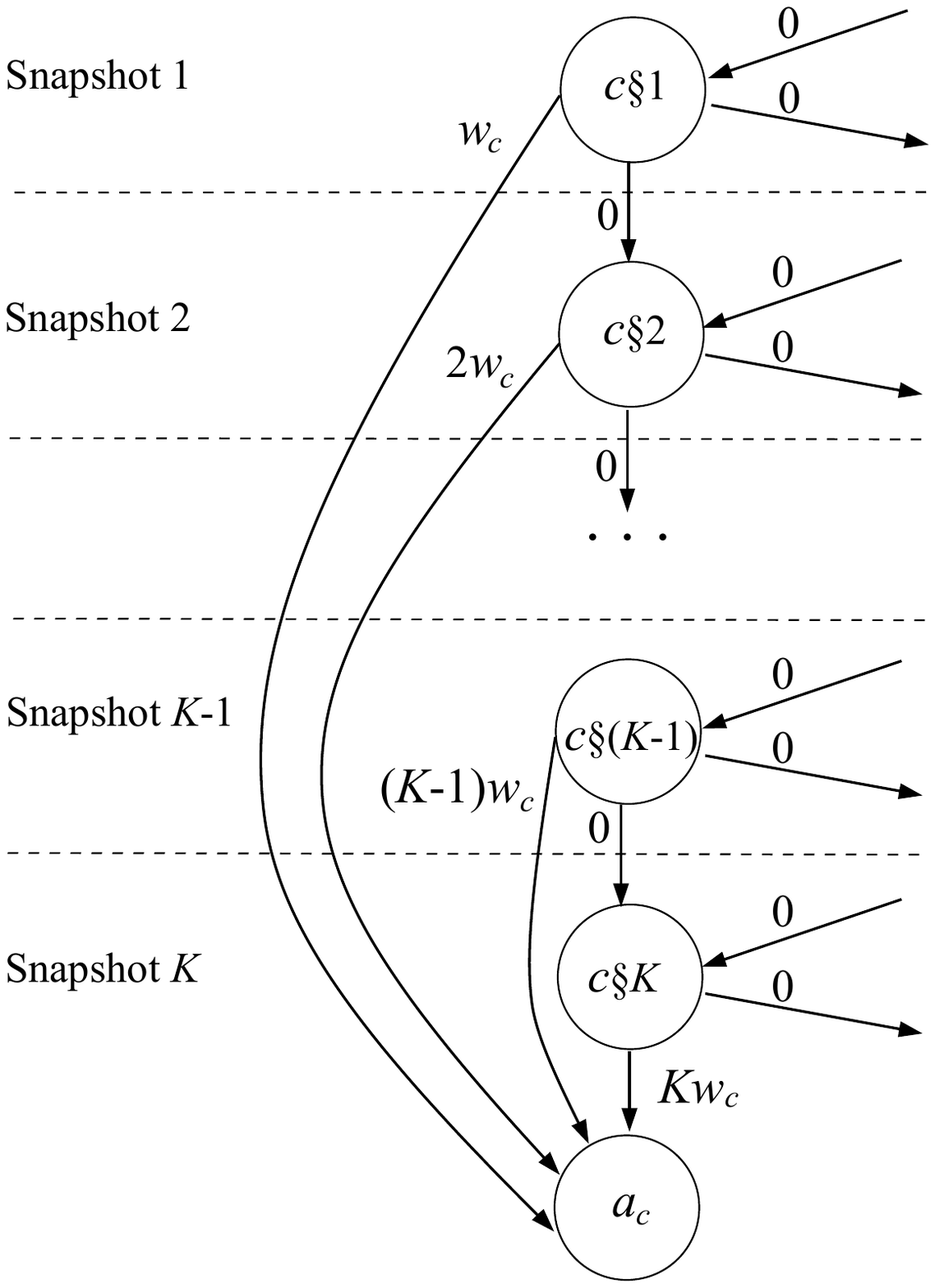} \caption{An
  illustration of expanded topology of inter-snapshot arcs and the
  unit flow cost for solving 1-UF-WS.}  \label{fig:mcf}
\end{figure}

Since there is a single FL model and routing takes place on a single path
for each client, we assume without loss of generality that an arc's
capacity is normalized to how many times the model can be accommodated
on the arc, i.e., $\lfloor p_{uv} / q \rfloor$ for arc $(u,v)$.  
For each client $c$, the inter-snapshot arcs $(c \S k, a_c), k =1, \dots, K$ 
all have (normalized) capacity 1.0 in the expanded topology.

Let the server node
of the first snapshot, i.e., $s \S 1 $, be a source node with flow
supply $|\CC|$, and for each client $c \in
\CC$, $a_c$ is a sink node of flow demand $1.0$.
For client $c$, the unit flow cost of arc $(c \S k, a_c), k=1,\dots,
K$, is set to be $k w_c$, whereas all other arcs have a unit flow cost
of zero, as illustrated in Figure~\ref{fig:mcf}.

Consider solving the problem of minimum-cost flow with the defined
capacity, flow supply and demand, and unit flow costs.  As a classic
result of optimal network flow, because the arc capacities as well as
flow supply and demand values are all integer, at optimum the flow is
integer, even if we solve the minimum-cost flow problem as a
continuous linear program (LP) without explicitly imposing the
integrality requirement~\cite[Theorem~9.10]{ahuja1993network}.
Therefore, in the optimal flow, for each client $c$ and its sink node
$a_c$, node $a_c$ has a flow of 1.0 on exactly one of the incoming
arcs. Note that there are $K$ such incoming arcs, each having
the unit flow cost that equals the WS objective function value.  Also,
if the flow destined to $a_c$ arrives $c \S k$ with $k < K$, it
will be routed on arc $(c \S k, a_c)$ without being buffed because
doing so will increase the total cost. Moreover, as the unit flow cost
is zero for all intra-snapshot arcs as well as the arcs for buffering
on a satellite for the flow of other satellites, there will be no
additional cost.  Therefore, the optimal minimum-cost flow solution is
the optimum of 1-UF-WS.
\end{proof}

\begin{remark}
In the proof, the flow demand of client $c$ is set in node $a_c$. However, backtracking leads to the snapshot in which the model actually arrives client $c$.  $\Box$
\end{remark}

\begin{theorem}
\label{theo:1ufmm}
1-UF-MM is solvable in polynomial time.
\end{theorem}
\begin{proof}
As discussed earlier, 1-UF-MM can be solved by bi-section search over
the snapshots, and for a trial snapshot $k$ we need to determine if
there is a feasible flow for all clients until this snapshot. Thus we
consider the nodes and arcs of snapshots $1, 2, \dots, k$ of graph
$\CG$ (without the expansion used in the previous proof), with the
same definition of capacity as in the previous proof.  Server $s$
remains the source of supply $|\CC|$, and node $a_c$ is an added sink (after snapshot $K$) of
demand 1.0 for each client $c$. Also, for each client $c$ we add an
arc $(s\S 1, a_c)$ with unit flow cost 1.0 and capacity 1.0.  The flow
cost of all other arcs are set to be zero.  Solving the minimum-cost
flow problem (which is always feasible due to the addition of the
direct arc from the server to each client's sink node). If the total
flow cost equals zero, i.e., no direct arc is used, routing is
feasible using the original arcs, otherwise it is infeasible.
\end{proof}

Since the MM objective function boils down to determining the
existence of feasible flow, it is straightforward to see that, with
splittable flow, this can be done using a similar idea as in the proof
of Theorem~\ref{theo:1ufmm} without normalizing the capacity nor
flow supply or demand.  Checking feasibility then amounts to solving
an LP for optimal (splittable flow), giving the following corollary.

\begin{corollary}
1-SF-MM is solvable in polynomial time.
\end{corollary}

Next, one observes that for two models, feasibility of splittable flow
can be determined in a similar manner, by setting a flow supply
of $q_1 |\CC_1| + q_2 |\CC_2|$ at server node $s \S 1$, and let
a client' sink node have a flow demand of the corresponding
model size (or the sum if the client is in $\CC_1 \cap \CC_2$).
With such a construction, it is easy to see
2-SF-MM is tractable as well, giving the corollary below.


\begin{corollary}
\label{theo:2sfmm}
2-SF-MM is solvable in polynomial time.
\end{corollary}

\subsection{NP-complete and NP-Hard Problem Variants}

We now prove that each of the remaining four problem variants in
Table~\ref{tab:downloading} are NP-hard.  We begin with
1-SF-WS that is the only NP-hard variant for distributing a single FL
model. In this case, the flow may arrive at a client over multiple
snapshots. By the WS objective function, for one client, only the last
snapshot index is of significance in performance, even if the client
would have received most parts of the FL model in earlier snapshots. One may argue
however this does make sense, since the client need the entire model
before its local training can start.

\begin{theorem}
\label{theo:1sfws}
1-SF-WS is NP-hard.
\end{theorem}
\begin{proof}
Consider a generic instance of 3SAT as defined in
Section~\ref{sec:3sat}.  The reduction uses the following special case
of 1-SF-WS.  The model size $q = 1.0$, and there are two snapshots.
For the first snapshot, the graph topology relates to the 3SAT
instance as follows, with an illustration given in
Figure~\ref{fig:1sfws}.

\begin{figure}
\centering
  \includegraphics[width=0.48\textwidth]{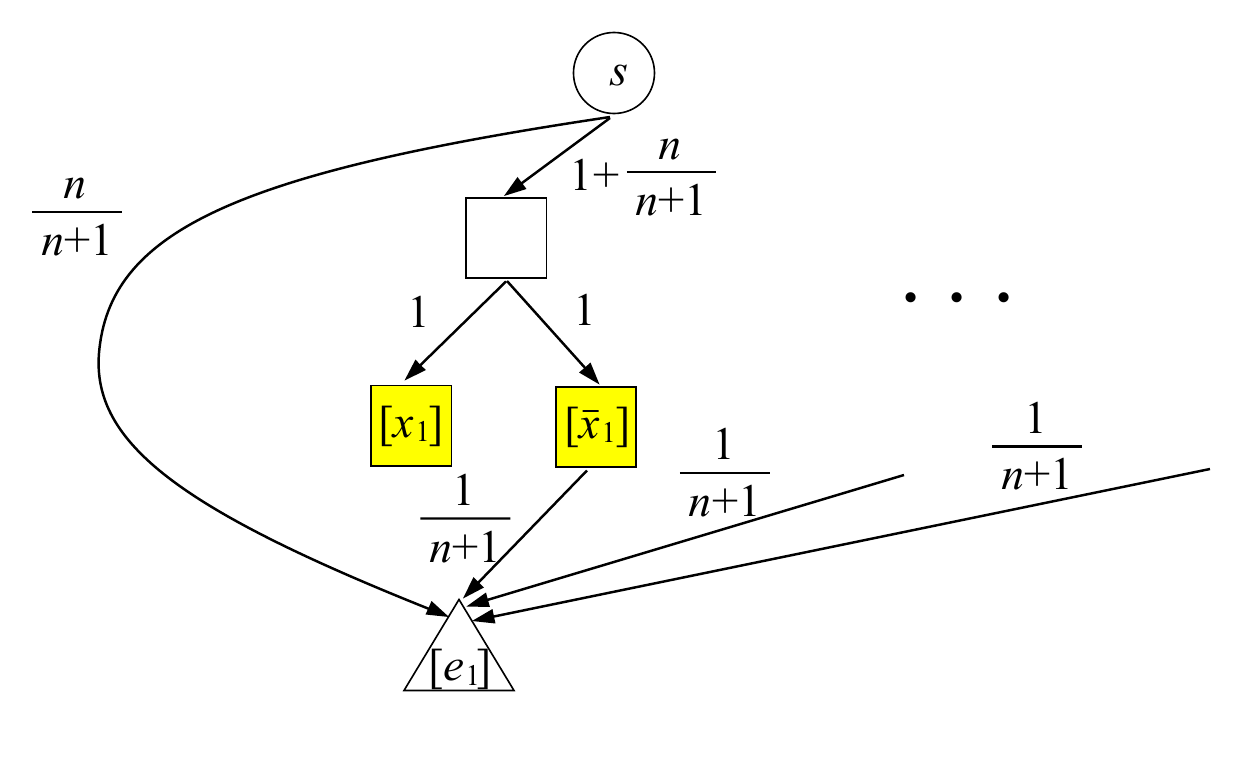}
  \caption{An illustration of reduction from 3SAT to 1-SF-WS. ($[\cdot]$ inside a node
indicates the corresponding entity of this node in 3SAT. The numbers beside the arcs
are the capacity values.)}
  \label{fig:1sfws}
\end{figure}

As can be seen from the figure, for each Boolean variable $x_\ell$, three
(square) nodes are present. There is one arc from server $s$ to one
square node of capacity $1 + \frac{1}{n}$.  This node is connected to
each of the other two square nodes (in yellow color) with an arc. The
capacity of these two arcs equals $1.0$. In the proof, the two yellow
nodes correspond to a Boolean variable and its negation ($x_1$ and
${\bar x}_1$ in the illustration), and they are referred to as the
literal nodes.  There is one (triangle) node for each clause (such as
clause $e_1$ in the illustration). For a clause node, there is a
direct arc from the server with capacity $\frac{n}{n+1}$.  Moreover,
there are exactly three additional incoming arcs to a clause, emanating from the
three literal nodes representing the negation of the three literals of
this clause in the 3SAT instance, with capacity $\frac{1}{n+1}$. In
the illustration, for example, literal $x_1$ is in clause $e_1$ in the 3SAT instance,
therefore the arc from the the negation ${\bar x}_1$ to the clause
node.

The set of clients is composed by the two yellow literal nodes for
each Boolean variable, and the clause node of each clause. There are
thus $2m+n$ clients. All clients are of unit weight. For the second
snapshot, the $4n$ arcs coming into the clause nodes are no longer
present.

We ask the question: For the given 1-SF-WS instance, is it feasible
to route the FL model to all the $2m+n$ clients within the two snapshots?
Note that for each pair of literal nodes, it is not possible to route
the entire flow of two units to them in the first snapshot.  Hence,
for such a client pair, two snapshots are needed, such that one will
have a WS objective function value of 1.0, and other 2.0, or vice versa. Since there is no way
to achieve a better value of 3.0 for these two together, without
loss of optimality, we can assume that the server sends the whole model of
size 1.0 to one of them in the first snapshot and zero flow to the
other, and the latter receives a flow of 1.0 in the second
snapshot. Doing so leaves the maximum spare capacity of
$\frac{n}{n+1}$ for routing in the first snapshot to the clause nodes
that have no connection at all in the second snapshot.

For a clause node, it receives a flow of amount $\frac{n}{n+1}$ via
the direct (and dedicated) arc from the server in the first snapshot.
Hence the routing is complete for a clause node, if and only if it can
receive an additional amount of $\frac{1}{n+1}$ in the first snapshot via
some literal node.

Suppose in the first snapshot, for literal pair $x_\ell$ and ${\bar
x}_\ell$, $x_\ell$ is selected to receive the FL model. Then the server can
send additional flow (up to $\frac{n}{n+1}$) via the literal node of
${\bar x}_\ell$ to some clause clients.  By the construction, these
correspond to the 3SAT clauses that contain $x_\ell$ as a literal. Since
there are $n$ clauses, all these clause clients can receive a flow of
$\frac{1}{n+1}$ as what is required to obtain the entire model in
snapshot one. If ${\bar x}_\ell$ is selected to receive the FL model
instead of $x_\ell$ in snapshot one, then all clause clients
having literal ${\bar x}_\ell$ can receive the entire model as they
are connected with $x_\ell$. If for a clause client,
there are multiple options of literal nodes
for delivering a flow of $\frac{1}{n+1}$, then any of them can be chosen (or the flow
can be arbitrarily split among them). 

From the above, it is now clear that whether or not client $x_\ell$ is
selected for receiving the model corresponds to a true/false value to
Boolean variable $x_\ell$. If there is a selection for all literal pairs, such
that all clause clients can receive the FL model in snapshot one,
there is a direct mapping to a positive answer to the 3SAT 
instance. Conversely, if there is a feasible solution to the 3SAT instances,
it can be easily mapped to a 1-SF-WS solution to satisfy all clients within the
two snapshots.
Hence the decision version of 1-SF-WS is NP-complete, and consequently 1-SF-WS
is NP-hard.
\end{proof}

Via Theorem~\ref{theo:1sfws}, we can immediately conclude the complexity of
2-SF-WS. By simply considering instances of 2-SF-WS where the two client
sets $\CC_1$ and $\CC_2$ do not overlap, and each set has its independent graph
topology as constructed in the proof above, 
it follows as a corollary that 2-SF-WS is NP-hard.

\begin{corollary}
\label{theo:2sfws}
2-SF-WS is NP-hard.
\end{corollary}

Next, we consider 2-UF-MM. Recall that with the MM objective, 2-UF-MM
amounts to answering the feasibility of routing two FL models to their
(respective) clients such that a single path is used for each client.
We prove this feasibility check is NP-complete and hence 2-UF-MM is
NP-hard, again using a reduction from 3SAT though with a different
construction than earlier.

\begin{theorem}
\label{theo:2ufmm}
2-UF-MM is NP-hard.
\end{theorem}
\begin{proof}
We construct a special case of 2-UF-MM such that whether or not
routing of two FL models can be completed within a single snapshot
gives the correct answer to a generic 3SAT instance.  An illustration
of the snapshot's topology is given in Figure~\ref{fig:2ufmm}.

\begin{figure}
\centering
  \includegraphics[width=0.35\textwidth]{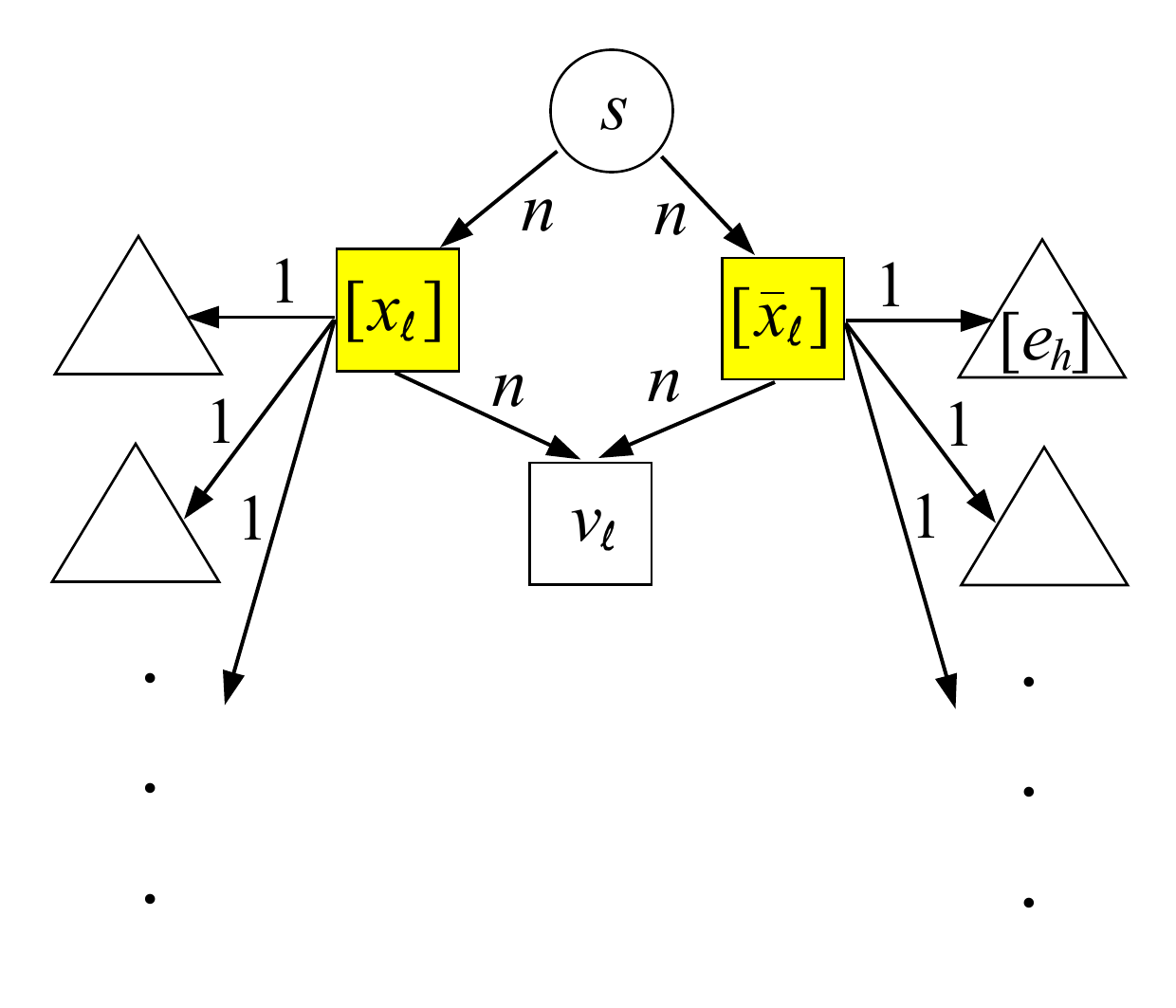} \caption{An
  illustration of reduction from 3SAT to 2-UF-MM. ($[\cdot]$ inside a
  node indicates the corresponding literal and its occurrences in the
  clauses in 3SAT. The numbers beside the arcs are the capacity
  values.)}  \label{fig:2ufmm}
\end{figure}

Suppose the sizes of the two models are $n$ and $1$, respectively.  We
refer to them as the big and small models, respectively, to simplify
the wording. For each Boolean variable $x_\ell$, we introduce two nodes
as shown in the figure. For $x_\ell$, there is an arc from the server $s$
and then an arc from $x_{\ell}$ to a client $v_\ell$ of the big model, both of
capacity $n$.  The same construction applies to ${\bar x }_\ell$.  This
topology is constructed for every pair of literals in 3SAT.  There are
thus $m$ clients of the big model.

In addition, the graph has $n$ nodes (of triangle shape in the
illustration) corresponding to the clauses in 3SAT. These are client
nodes of the small model.  For each clause node, there is an arc  to it from the node representing each literal contained in the clause, with capacity 1.0. 

For each client $v_\ell$ of the big model, there are two paths via the
two literal nodes $x_\ell$ and ${\bar x}_\ell$, respectively.
Choosing the path of $x_\ell$ corresponds to setting value true to
Boolean variable $\bar x_\ell$, namely the negation of $x_\ell$, in 3SAT. In this case, there is no spare
capacity to route any flow to the clause nodes connected to $x_\ell$.
But the server can route the small model to all the clause clients
connected to ${\bar x}_\ell$. At this point, it is clear that the routing for this 2-UF-MM
is feasible within one snapshot if and only if the answer to the 3SAT instance is true,
and the theorem follows.
\end{proof}

In light of Theorem~\ref{theo:2ufmm}, it is rather intuitive that
changing the objective to WS should not alter the problem complexity
since optimizing WS cannot be easier than feasibility check.  Indeed, a
formal proof is straightforward based on the proof of
Theorem~\ref{theo:2ufmm}, and therefore we state the result as a
corollary.

\begin{corollary}
2-UF-WS is NP-hard.
\end{corollary}

Now we have shown the complexity of all problem variants in
Table~\ref{tab:downloading}.  However, for the settings with two
models, it has been assumed they share a common server.
What if separate servers are used? The answer is that the complexity results remain.
More specifically, using the arguments for concluding
Corollary~\ref{theo:2sfmm} and Corollary~\ref{theo:2sfws},
it is easy to see that 2-SF-MM remains tractable
and 2-SF-WS remains NP-hard with two separate servers.
For 2-UF-MM, an NP-hardness proof can be derived by 
adding two server nodes $s_1$ and $s_2$, and two arcs to connect to the original server node $s$ (now acting as a regular satellite) in the topology 
in Figure~\ref{fig:2ufmm},
with capacity $mn$ and $n$, respectively, for the big and small models.
Then the line of proof for Theorem~\ref{theo:2ufmm} remains valid.
Finally, for 2-UF-WS the hardness result follows easily using a similar proof.
We summarize the observations below.

\begin{theorem}
The complexity results of 2-SF-MM, 2-SF-WS, 2-UF-MM, and 2-UF-WS remain
for two separate servers.
\end{theorem}

The NP-hard problem variants for two FL models clearly
remain NP-hard if there are more than two FL models. 
For 2-SF-MM that is polynomial-time solvable, generalizing
this complexity result to three or more models 
consists in the underlying rationale for Corollary~\ref{theo:2sfmm}.

\section{Global Model Distribution: Multicast}
\label{sec:multicast}

Since all clients of an FL model are receiving the same model
parameters in the downloading phase, deploying multicast routing
becomes attractive to save bandwidth. With multicast, only one copy
of the global model is routed. Generally speaking, multicast routing with unsplittable flow uses a tree structure rooted at
the source, where each node distributes the same data to its children
in the tree.  Constructing a multicast tree for routing was studied
very early~\cite{12889}. The basic setting leads to constructing a Steiner tree
(or Steiner arborescence for directed graphs) where the so called
terminal nodes correspond to the clients in our context. Steiner tree
is a classic NP-hard problem~\cite{WiHwRi92}. We will see, however,
this result does not generally imply NP-hardness for multicast routing
for FL model downloading. Multicast with splittable flow, however, will not be discussed here. Although it is technically feasible in SR, this design involves increased complexity in determining the number of trees and generating multiple trees where minimal path overlap is sought for effective load balancing, thus decreasing its value of practicality for in-orbit FL.\par

For multicast, we consider problem variants that differ in the
objective function (WS or MM) and the number of FL models.  We use
prefix ``mul'' to indicate multicast routing. The four problem settings
are mul-1-WS, mul-1-MM, mul-2-WS, and mul-2-MM. Their computational complexity is summarized in Table \ref{tab:multi-UF}.

\subsection{Multicast Routing for one FL model}
\begin{table}[!t]
\caption{Problem variants and their computational complexity for multicast in global model distribution. (UF = unsplittable flow, WS = weighted sum, MM = min-max, P = polynomial time.)}
\label{tab:multi-UF}
\centering
\begin{tabular}{c|cc}
\toprule
\multirow{2}{*}{\textbf{Models}} & \multicolumn{2}{c}{\textbf{UF}} \\ \cmidrule(lr){2-3}
 & \textbf{WS} & \textbf{MM} \\ \midrule
\textbf{1} & \begin{tabular}[c]{@{}c@{}}Mul-1-WS\\ \emph{(P)}\end{tabular} & \begin{tabular}[c]{@{}c@{}}Mul-1-MM\\ \emph{(P)}\end{tabular} \\
\textbf{2} & \begin{tabular}[c]{@{}c@{}}Mul-2-WS\\ \emph{(NP-hard)} \end{tabular} & \begin{tabular}[c]{@{}c@{}}Mul-2-MM\\ \emph{(NP-hard)}\end{tabular} \\ \hline
\end{tabular}
\end{table}

For one FL model, we can compare the arc capacities with
respect to the model size, such that an arc is present, if and only if
its capacity is no less than the model size.  Hence the capacity is
represented by graph connectivity.

We start with mul-1-MM, as concluding its complexity is 
easy. In mul-1-MM, the task is to answer the existence of a Steiner
arborescence 
rooted at the server $s$ (in the first snapshot, represented by $s \S 1$), spanning all clients in $\CC$
and possibly including some non-client nodes (the so called Steiner
nodes), over graph $\CG$ defined until a given snapshot $k$.
Obviously, such an arborescence exists, if an only if there is some
connectivity from $s$ to each client node $c$ in snapshot $k$, that is
node $c \S k$. Therefore, mul-1-MM amounts to simple graph search.


\begin{theorem}
\label{theo:mul1mm}
Mul-1-MM is polynomial-time solvable.
\end{theorem}

We now turn our attention to mul-1-WS, where the cost of a
client is associated with the snapshot in which the client receives
the model. Problem mul-1-WS does resemble the classic
Steiner arborescence problem. However, somewhat surprisingly, we will
prove that mul-1-WS is solvable in polynomial time.

For the proof, consider the expanded topology for
modeling inter-snapshot arcs of each client, as shown earlier in
Figure~\ref{fig:mcf}. Since we deploy multicast, the
interpretation of the arc costs in the figure is no longer unit flow
costs. Instead, the cost of an arc is the (one-time) cost of including
the arc in the arborescence. Note that for any client $c$, only arcs
entering $a_c$ carry positive costs, and this property will be utilized
in the proof of tractability.

For the Steiner arborescence to be used for multicast routing, node
$a_c$ of each client $c$ is a terminal node (that is a node that must
be included in the arborescence). It is clear that $a_c$ has to be
reached via arc $(c \S k, a_c)$ for some $k, 1 \leq k \leq K$, with
cost $k w_c$, and there is no other cost incurred.  Therefore, the
total cost of the Steiner arborescence indeed is what is intended for
mul-1-WS.


Before proceeding, we assume that $s$ can reach all terminal
nodes and hence there exists an arborescence rooted at node $s \S 1$ spanning all
$a_c, c \in \CC$, as otherwise there is nothing to optimize.
(Checking if the assumption holds is straightforward as it amounts to
simple graph search.)  Let $\CT^M = \{a_c, c \in \CC\}$ denote the set
of terminal nodes.  For the Steiner nodes, i.e., those in $\CV
\setminus (\CT^M \cup\{s\})$, some may not be reachable at all from server $s$.
Obviously, such nodes are fully redundant, and hence we
discard them from the graph.  We use notation $\CG^M = (\CV^M, \CA^M)$
for the graph obtained from the original graph $\CG$ with the topology
modification as illustrated in Figure~\ref{fig:mcf} and the removal
of non-reachable Steiner nodes.

\begin{lemma}
\label{theo:sarboresence}
For graph $\CG^M = (\CV^M, \CA^M)$, denote by $\CH^*$ the minimum-cost
Steiner arborescence rooted at $s \S 1$ spanning all terminals in $\CT^M$.
A node is a leaf in $\CH^*$ if and only if the node is in $\CT^M$.
\end{lemma}
\begin{proof}
The lemma follows immediately from that $\CH^*$ spans all nodes in
$\CT^M$ and there is no outgoing arc from any node in $\CT^M$ by graph
construction.
\end{proof}

\begin{theorem}
Mul-1-WS is polynomial-time solvable by computing a minimum-cost
arborescence on $\CG^M$.
\end{theorem}
\begin{proof}
Consider $\CH^*$ that is an optimal Steiner arborescence.  By
Lemma~\ref{theo:sarboresence}, all costs occur on the incoming arcs to
the leave nodes in $\CT^M$.  Denote the subset of nodes in $\CH^*$ by
$\CV_{\CH^*}$.  Because all nodes are reachable from $s \S 1$ in
$\CG^M$, there must exist at least one arc connecting some node in
$\CV_{\CH^*}$ to some node in $\CV^M \setminus \CV_{\CH^*}$. Note that
both ends of the arc are non-terminal nodes (i.e., either the server
or Steiner nodes), since all terminals in $\CT^M$ are in $\CV_{\CH^*}$
and there is no outgoing arc from any node in $\CT^M$. Consider adding
any of such arcs and its end node to $\CH^*$. The tree structure is
clearly maintained and no additional cost is incurred. This expansion
can be repeated with zero increase in cost, until all nodes are
included in $\CH^*$. At this stage, we obtain an arborescence spanning
all nodes.  Since a minimum-cost arborescence has a total cost that
cannot be lower than that of minimum-cost Steiner arborescence, and
the former is obtained in polynomial time with the Chu-Liu/Edmonds'
algorithm~\cite{Edmonds}, the theorem follows.
\end{proof}

\subsection{Multicast Routing for Two FL Models}

Consider mul-2-MM where the two FL models have the same server.  Here
we ask for two arborescences, both rooted at node $s \S 1$, such that
they span the client nodes of the two models, respectively, in the
graph topology defined for up to the snapshot in question.  An arc may
be present in both arborescences, if its capacity can accommodate both
models.  For one model, mul-1-MM is easy since it simply reduces to
graph connectivity, giving Theorem~\ref{theo:mul1mm}. For two models,
we show that answering the question of feasibility becomes
NP-complete, and hence mul-2-MM is NP-hard. 

\begin{theorem}
Mul-2-MM is NP-hard.
\end{theorem}
\begin{proof}
A proof can be established using again the reduction from 3SAT in the
proof of Theorem~\ref{theo:2ufmm}, where a graph is defined for a
single snapshot, see also Figure~\ref{fig:2ufmm}.  For each client
$v_\ell$ of the big model, it is apparent that either the left yellow
node ($x_\ell$) or the right yellow node (${\bar x}_\ell$) must be
chosen. This gives a directed path from $s$ to $v_\ell$. As in the
proof of Theorem~\ref{theo:2ufmm}, a clause client can be connected to
one of its three yellow nodes, unless all are blocked due to the routing
choice made for the big model.  Moreover, sending copies of the small
model via a yellow node to multiple clause clients implies that the
multicast mode can be deployed instead, and vice versa. Thus there is
an equivalence of feasibility of 2-UF-MM and mul-2-MM for the graph
constructed in the proof. In fact, a feasible solution of 2-UF-MM in
the graph is composed by two arborescences.  Hence mul-2-MM is NP-hard.
\end{proof}

As for the case of 1-UF-MM and 1-UF-WS, optimizing the WS objective
function is not easier than determining feasibility with two FL
models, a result that is easily obtained by adapting the proof of
mul-2-MM for mul-2-WS. We thus have the following corollary.

\begin{corollary}
Mul-2-WS is NP-hard.
\end{corollary}

We end the section by examining multicast routing of two models
with different servers. As expected, the problem does 
not become easier. In fact the complexity result remains
even if the models are of the same size.

\begin{theorem}
\label{theo:mul2mm}
The separate-server version of mul-2-MM and that of mul-2-WS are
NP-hard, even if the two models are of the same size.
\end{theorem}
\begin{proof}
We provide a reduction from the two edge-disjoint paths
problem in directed graphs that is NP-complete~\cite{FoHoWy80}.
In this problem, we are given a directed graph and two separate pairs of 
origin-destination nodes, $o_1$ and $d_1$, and $o_2$ and $d_2$.
The task is to determine if there are two edge-disjoint paths
connecting the two node pairs, respectively.

A reduction to mul-2-MM is fairly easy. Given the directed graph in
the two edge-disjoint paths problem, we use the same topology as a
single snapshot in mul-2-MM. Let $s_1 = o_1$ and $s_2 = o_2$.  Both
models are of size one and all arc capacities are set to be one.  In
addition, suppose $d_1$ is the only client of the first model, and
$d_2$ is the only client of the second model. Note that for each model
we ask for a single directed path due to the presence of one client,
and a directed path is a special case of arborescence.  Then
determining if multicast routing for the mul-2-MM instance can be done in
one snapshot or not gives the correct answer of the two edge-disjoint paths
problem.  The theorem then follows.
\end{proof}

\section{Model Collection and Client Selection}
\label{sec:uploading}

Let us examine the tractability aspect of uploading the (updated) local models. Similar to downloading, we have the options of using SF or UF in uploading the model of each client, and scenarios with one and two models. However, since the local models differ from each other, multicast is not applicable. Moreover, due to limited communication resource, it may not be feasible to collect the local models from all clients by snapshot $K$. In this case, CS becomes necessary. Since not all clients contribute equally to the aggregation of the global model, clients are selected based on utility maximization (recall Section \ref{sec:obj}). When there are two FL models, we assume a common server and comment later on the case of two separate servers. The problem variants and their tractability results are summarized in Table~\ref{tab:uploading}. \par
To avoid additional notation, in the following, we reuse $\CK$ for the set of snapshots in uploading, as doing so does not cause ambiguity. Consequently we reuse notation $\CG$ for the underlying TVG for uploading, even though $\CK$ and $\CG$ are different from those for the downloading phase. 

\begin{table}[t]
\centering
\caption{Problem variants and their computational complexity for mode collection.
(UF = unsplittable flow, SF = splittable flow, CS = client selection,
P = polynomial time.)}
\begin{tabular}{c|cc|cc}
\toprule
\multirow{2}{*}{\textbf{Models}} &
\multicolumn{2}{c|}{\textbf{UF}} &
\multicolumn{2}{c}{\textbf{SF}} \\
\cmidrule(lr){2-3} \cmidrule(lr){4-5}
 & \textbf{No CS} & \textbf{CS} & \textbf{No CS} & \textbf{CS} \\
\midrule
\textbf{1} &
\begin{tabular}[c]{@{}c@{}}1-UF-NCS\\ \textit{(P)}\end{tabular} &
\begin{tabular}[c]{@{}c@{}}1-UF-CS\\ \textit{(P)}\end{tabular} &
\begin{tabular}[c]{@{}c@{}}1-SF-NCS\\ \textit{(P)}\end{tabular} &
\begin{tabular}[c]{@{}c@{}}1-SF-CS\\ \textit{(NP-hard)}\end{tabular} \\[4pt]
\textbf{2} &
\begin{tabular}[c]{@{}c@{}}2-UF-NCS\\ \textit{(NP-complete)}\end{tabular} &
\begin{tabular}[c]{@{}c@{}}2-UF-CS\\ \textit{(NP-hard)}\end{tabular} &
\begin{tabular}[c]{@{}c@{}}2-SF-NCS\\ \textit{(P)}\end{tabular} &
\begin{tabular}[c]{@{}c@{}}2-SF-CS\\ \textit{(NP-hard)}\end{tabular} \\
\bottomrule
\end{tabular}
\label{tab:uploading}
\end{table}


\subsection{Tractable Problem Variants}

We begin with the tractable problem variants. With one model, if we do
not need to make client selection and the flow of a client can be
split on multiple paths, the problem reduces to determining the
feasibility of flow in a graph with given source and destination nodes.
This is classic network flow problem.

\begin{theorem}
1-SF-NCS is polynomial-time solvable.
\end{theorem}
\begin{proof}
For graph $\CG = (\CV, \CA)$ with capacity $p_{uv}$ on arc $(u,v)$,
let each $c \S \kappa_c, c \in \CC$ be a source node of flow supply $q$, and $s \S
K$ the destination node with flow demand $q |\CC|$. Consider introducing
an auxiliary variable $\rho_{uv}, (u,v) \in \CA$, and treat $p_{uv} +
\rho_{uv}$ as the capacity, with the objective function of minimizing
$\sum_{(u,v) \in \CA} \rho_{uv}$, subject to flow conservation. This
is a rather standard optimal network flow
problem~\cite{ahuja1993network} that can be formulated as an LP.
Clearly, whether or not the optimum value is zero answers the question
of feasibility, and the theorem follows.
\end{proof}

Suppose there are two FL models. The above idea remains valid for
2-SF-NCS, by letting the nodes representing clients in $\CC_1$ and
$\CC_2$ in snapshot one be the sources nodes, and $q_1$ and $q_2$ be
the amount of supply, respectively. If a client has both models to
upload, its supply is set to be $q_1 + q_2$. This leads to, as above,
an optimal network flow problem that can be modelled as an LP,
giving the following result.

\begin{corollary}
2-SF-NCS is polynomial-time solvable.
\end{corollary}

For 1-UF-NCS, it has a similar structure as 1-UF-MM, simply because of
the MM objective in uploading. Utilizing the same type of arguments for
proving Theorem~\ref{theo:1ufmm}, we obtain the
following result.

\begin{theorem}
1-UF-NCS is polynomial-time solvable.
\label{theo:1-UF-NCS}
\end{theorem}
\begin{proof}
Similar to the proofs of Theorems~\ref{theo:1ufws} and
\ref{theo:1ufmm}, we can normalize the arc capacity $p_{uv}, (u,v) \in
\CA$ with respect to the model size $q$, to obtain integer capacity
values.  Let each node $c \S \kappa_c, c \in
\CC$ be a source with supply 1.0, and node $s \S K$ be the sink node
of demand $|\CC|$. For each $c \in \CC$, a direct arc $(c \S \kappa_c, s \S K)$
with unit flow cost of 1.0 is added. The unit flow cost
of all original arcs are set to be zero. The resulting minimum-cost
flow problem has integrality property, hence the conclusion.
\end{proof}

Now consider 1-UF-CS. This problem can be solved within polynomial time by utilizing the feature of uniform demand of sources and the integer property of a minimum cost flow problem. 

\begin{theorem}
1-UF-CS is polynomial-time solvable.
\end{theorem}
\begin{proof} 
For a TVG $\CG = (\CV, \CA)$, similar to the proofs of Theorems~\ref{theo:1ufws}, \ref{theo:1ufmm}, and \ref{theo:1-UF-NCS}, we can normalize the arc capacity $p_{uv}, (u,v) \in \CA$ with respect to the model size $q$, to obtain integer capacity values. The cost of all arcs in $\CA$ is set to zero. We then add an auxiliary node $a$ and one arc from it to each candidate source node $c \S \kappa_c, c \in \CC$. The arc from $a$ to $c \S \kappa_c$ has capacity one and cost $-u_c$ where $u_c$ is the utility for client $c$. Besides, one arc from node $a$ to node $s \S K$ with a sufficiently large capacity ($\ge |\CC|$) and zero utility is added. Since all arc capacities are integers, there exist an integer solution at the optimum of the minimum cost problem with source $a$, sink $s \S K$, and demand $|\CC|$. Note that this problem is always feasible due to arc $(a, s \S K)$. In addition, the integrality property also holds for negative cost, and in our graph there is no negative cycle. Furthermore, since the arc between $a$ and any $c \S \kappa_c, c \in \CC$ has capacity one, there exists flow of either size one or zero going through each candidate client $c \S \kappa_c$. The clients with flow of size one going through them are thus the optimally selected clients. 
\end{proof}

\subsection{NP-complete Problem Variants}

Without the need of CS, the problem variant that remains to be examined is
2-UF-NCS. It has in fact a similar structure to 2-UF-MM for
downloading. By considering the reduction proof for
Theorem~\ref{theo:2ufmm} and modifying the graph in the proof
by setting the arcs in the reserve direction, we obtain the following
theorem.

\begin{theorem}
2-UF-NCS is NP-complete.
\end{theorem}
\begin{proof}

We reuse the notion and reduction in the proof of
Theorem~\ref{theo:2ufmm}, but with reverted arc
directions, see Figure~\ref{fig:2ufncs}.

\begin{figure}
\centering
  \includegraphics[width=0.37\textwidth]{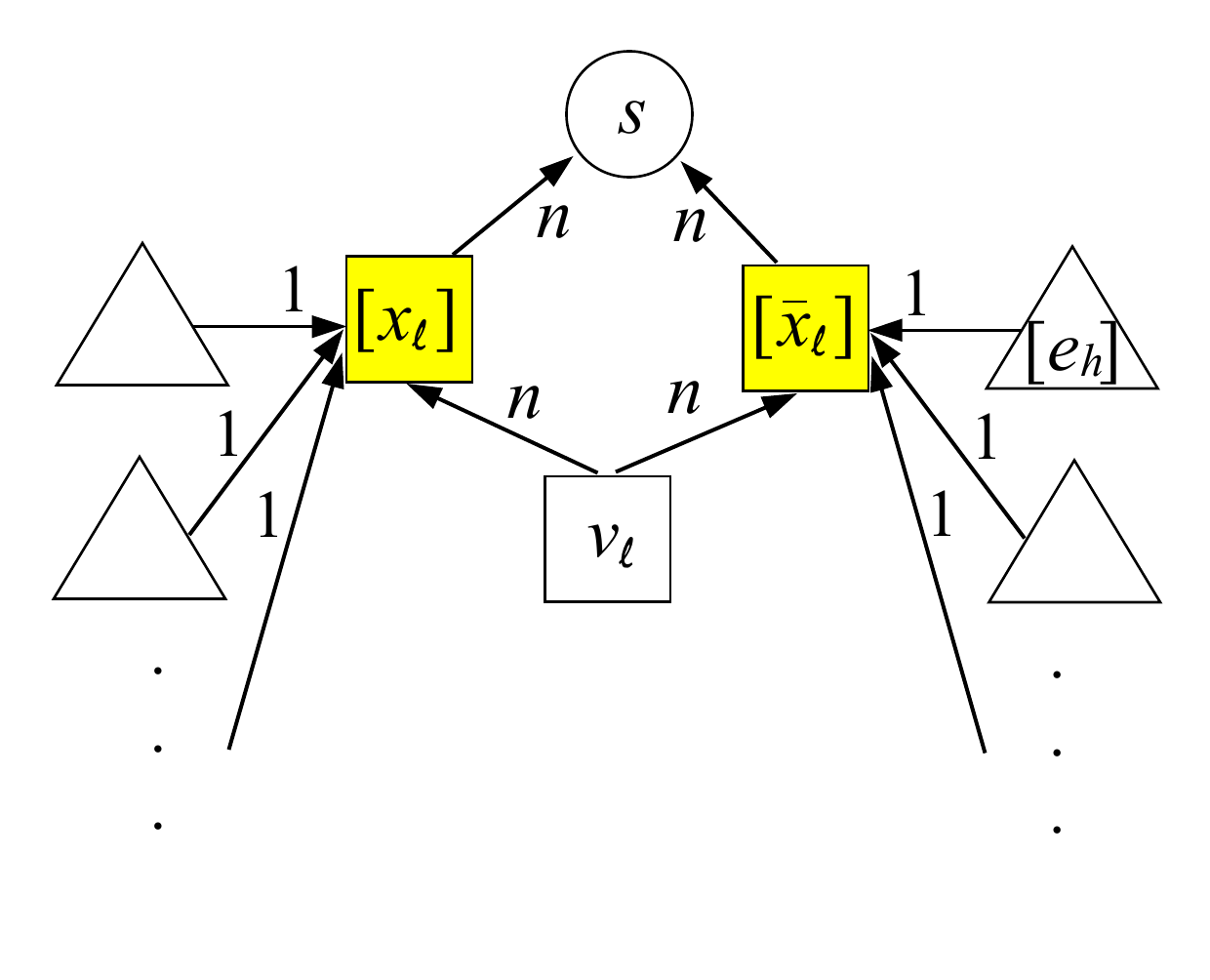} \caption{An
  illustration of reduction from 3SAT to 2-UF-NCS. ($[\cdot]$ inside a
  node indicates the corresponding literal and its occurrence in the
  clauses in 3SAT. The numbers beside the arcs are the capacity
  values.)}  \label{fig:2ufncs}
\end{figure}

The big model from client $v_\ell$ to server $s$ must be routed either
via literal node $x_\ell$ or literal node ${\bar x}_\ell$. In either case,
the capacity of the selected path becomes exhausted, whereas all three
clause clients connected to the literal node of the other path can
send their models to the server. 
Hence, except that the flow is on the opposite direction,
we obtain the same type of reduction from 3SAT as in the proof of Theorem~\ref{theo:2ufmm},
and the result follows.
\end{proof}

Next, we turn our attention to the scenario where the network resource
does not admit model uploading for all clients by the last snapshot
(i.e., snapshot $K$), and therefore CS is present. The question of
feasibility is thus to determine whether or not the local models of
the clients can be routed by snapshot $K$, subject to some CS target.
The CS target can be, for example, if it is feasible to select a given
number of clients such that their models can all be routed.  As
another example, we may ask if there is a CS solution such that
routing is feasible, and the total local loss of the selected clients
does not exceed some given threshold. We will show the this task is
NP-complete even for a single FL model.  Hence, optimizing the
selection of clients, no matter the objective function, is NP-hard.

\begin{theorem}
1-SF-CS is NP-hard.
\end{theorem}
\begin{proof} 
We provide a reduction proof via minimum-vertex cover (MVC). In MVC,
we are given an undirected graph $\CG^\text{MVC} = (\CV^\text{MVC},
\CE^\text{MVC})$, and the task is to select a minimum number of nodes
of $\CV^\text{MVC}$, such that all edges are adjacent to at least one
of the selected nodes (that is, covered by at least one node).  Let
$V^\text{MVC} = |\CV^\text{MVC}|$ and $E^\text{MVC} = |\CE^\text{MVC}|$.
Clearly, MVC can be solved by bi-section search on interval $[1,
V^\text{MVC}]$, and for each trial number $n^\text{MVC}$, answering the
feasibility question: Can all edges be covered by using $n^\text{MVC}$
nodes?  Since MVC is NP-hard, this feasibility problem is NP-complete.

We reduce the above feasibility version of MVC to a specific instance
of 1-SF-CS with a single snapshot. An illustration is given in
Figure~\ref{fig:1sfcs}. The graph of the 1-SF-CS instance is
constructed as follows. For each node $i \in \CV^\text{MVC}$, there is
a node that is a client; these are the nodes of the top layer in the
illustration.  We refer to such a node as an N-node. For each edge
$(i,j) \in \CE^\text{MVC}$, there is a node corresponding to this
edge. We refer to such a node as an E-node. The E-nodes are the ones
in square shape in the illustration.  For each E-node, there are
exactly two incoming arcs from the two N-nodes constructed for
the two nodes of the corresponding edge in the MVC instance. These arcs
have capacity $E^\text{MVC}$.  From each E-node, there is one incoming arc to
the server node $s$, with capacity 1.0. There is an auxiliary node
$a$, with one arc from each N-node of capacity 
$E^\text{MVC}$. Finally, there is an arc $(a, s)$ of capacity 
$(n^\text{MVC} - 1) E^\text{MVC}$. The model size $q = E^\text{MVC}$.

\begin{figure}
\centering
  \includegraphics[width=0.505\textwidth]{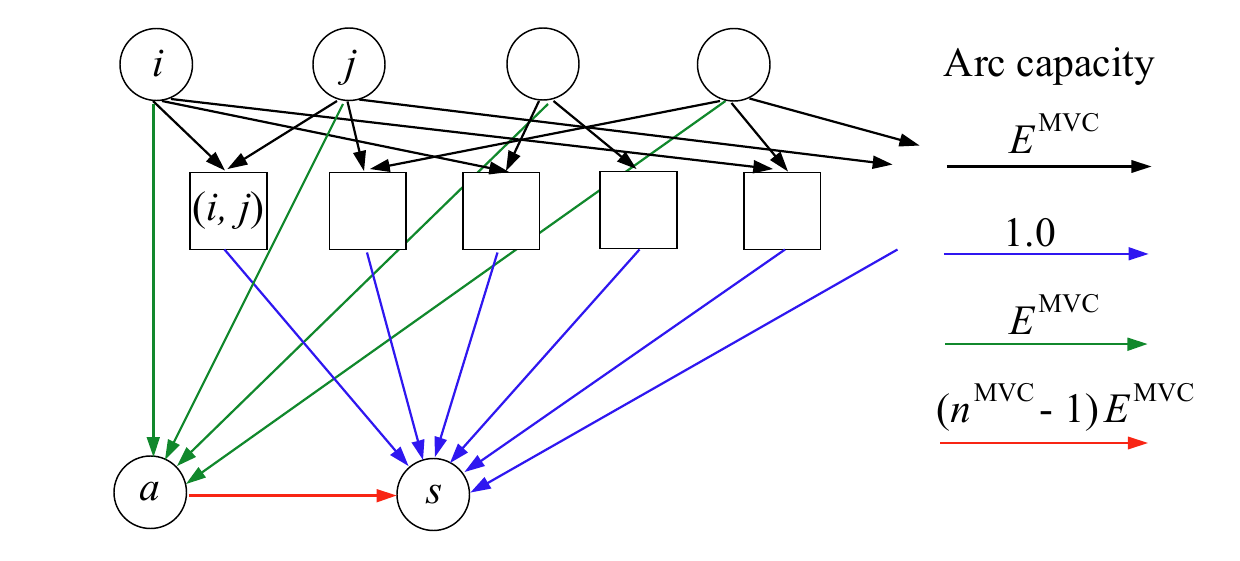} \caption{An
  illustration of reduction from MVC to 1-SF-CS.}  \label{fig:1sfcs}
\end{figure}

For the 1-SF-CS instance defined above, we ask the question: Is it
possible to select $n^\text{MVC}$ clients such that there is a
feasible routing solution for them? We will prove that
this question is equivalent to the feasibility question for MVC.

Suppose it is feasible to cover all edges of $\CE^\text{MVC}$ with
$n^\text{MVC}$ nodes in $\CV^\text{MVC}$.  Let such a node subset be
$\CV^*$, which is also referred to as a cover.  We associate each node
$i \in \CV^*$ with a number, obtained as follows. First, we find the
node covering the largest number of edges, and this number is the
number that the node is associated with. Tie, if any, is broken
arbitrarily.  These edges are then removed, and the node covering the
largest number of the remaining edges is found. The process is
repeated until no edge remains (which will eventually happen since
$\CV^*$ is a cover).  At this stage, zero is the value given to
remaining nodes.  Denote this number by $\eta_i, i \in \CV^*$.

Consider selecting the N-nodes corresponding to $\CV^*$ as the clients
subject to routing.
For each $i \in \CV^*$, we let a flow of $\eta_i$ to be routed to
server $s$ via the E-nodes, and specifically, via the E-nodes
representing the edges covered by $i$ in the above process, such that
a unit flow is routed via each of the E-nodes.  Doing so is clearly
feasible by the construction of the graph and the capacity
parameters. The remaining amount of $E^\text{MVC} - \eta_i$ is routed
via the auxiliary node $a$. Note that the total capacity of all
incoming arcs to $s$ equals $n^\text{MVC} E^\text{MVC}$ which is exactly the
total amount of flow to be routed. Therefore, there is a feasible
routing solution for the $n^\text{MVC}$ selected clients.

Conversely, suppose there exist $n^\text{MVC}$ clients with a feasible
routing solution. Since the capacity of arc $(a,s)$ is $(n^\text{MVC}
- 1) E^\text{MVC}$, in the routing solution, an exact flow amount of
$E^\text{MVC}$ is routed via the E-nodes. Because the arc connecting an E-node to
server $s$ has unit capacity, all these arcs must be in use.
Consequently, there is incoming flow from some client to every E-node
in the routing solution. By construction, there is an arc from an N-node
to an E-node if and only if the corresponding node in 
$\CV^\text{MVC}$ covers the edge represented by the E-node in
$\CG^\text{MVC}$. Therefore, the $n^\text{MVC}$ clients
form a cover in $\CG^\text{MVC}$, and the theorem follows.
\end{proof} 
 
Having identified the complexity of 1-SF-CS, it is straightforward
to generalize it to two models. Obviously, we can apply the proof above
by duplicating the graph of the MVC problem, giving the corollary below.

\begin{corollary}
\label{theo:2sfcs}
2-SF-CS is NP-hard.
\end{corollary}

\begin{theorem}
2-UF-CS is NP-hard.
\end{theorem}
\begin{proof}
MAX-3SAT \cite{GareyJohnson1979} is a NP-complete problem, with the same setting of $m$ Boolean variables and $n$ clauses as 3SAT, but it aims to find an assignment of variables that satisfies as many clauses as possible. We prove 2-UF-CS is NP-complete via a reduction from the MAX-3SAT problem to a special case of 2-UF-CS with only one snapshot $K$. We re-ultilize the construction in Figure \ref{fig:2ufncs}, where the triangle nodes represent the clients of the small model, and the (white) rectangle nodes represent the clients of the big model. Each client of the small model and the big model has utility one and $n+1$ for CS, respectively. This topology is constructed for each Boolean variable, hence there are in total $n$ clients for the small model and  $m$ clients of the big model.\par
Suppose the sizes of the two models are $n$ and $1$, respectively. For each variable node (a yellow rectangle node), at most $n$ units of flow can pass through it to the server $s$. Due to that the utility (i.e., $n+1$) of a client of the big model is higher than the sum utility (i.e., $n$) of all clients of the small model. One can verify that at the optimum of the 2-UF-CS problem, all clients of the big model must be selected. Furthermore, due to unsplittable flow, one of the two paths from $v_\ell$ to $s$ via the two literal nodes $x_\ell$ and ${\bar x}_\ell$ must be fully occupied by the big model. Choosing the path of $x_\ell$ corresponds to setting value true to Boolean variable $\bar x_\ell$ in 3SAT. In this case, flow of the small model (if any) can only go through ${\bar x}_\ell$. Hence, determining whether at least $n'$ clauses can be satisfied in the MAX-3SAT instance is equivalent to determining whether the total utility achieved is at least $m(n+1)+n'$ in the 2-UF-CS instance.
\end{proof}

Suppose the two models have separate servers $s_1$ and
$s_2$. By introducing a new node $s$ with arcs $(s_1, s)$ and $(s_2,
s)$ with large capacity, it is easy to see that 2-SF-NCS remains
polynomial-time solvable. As for the other three problem variants
that are NP-complete, by using the same construction, 
one can immediately conclude they remain NP-complete
for the case of two separate servers.

\begin{corollary}
The complexity results of 2-UF-NCS, 2-UF-CS, 2-SF-NCS, and 2-SF-CS
remain for two separate servers.
\end{corollary}

Finally, for more than two FL models, the tractability of 2-SF-NCS
generalizes because it remains an LP. For NP-complete variants of
two FL models, obviously they remain NP-hard when scaling up further
the number of models.


\section{Conclusions}
\label{sec:conclusion}

We have analyzed the tractability landscape of optimal
routing for global model distribution and local model collection over
time-varying graphs for federated learning in satellite networks.  With unicast routing for
model distribution, the min-max objective leads to tractability,
except for routing of two models with unsplittable flow.  For the other objective of weighted sum, the situation is
largely reversed - the problem variants are NP-hard, except for the case of one model with unsplittable flow.  If we deploy multicast
routing for model distribution, optimal routing is tractable for one
FL model, but NP-hard for two models. For model collection, it is tractable if no client selection needs to be performed, except for the
case of two models with unsplittable flow. With client selection, all problem variants
are NP-hard, except for the case of one model with unsplittable flow. Overall, our work provides a systematic characterization of routing complexity that guides the principled routing design, evaluation, and deployment in satellite-based FL or similar distributed learning systems.\par

Fully decentralized asynchronous FL, where there is no central server and nodes exchange and aggregate updates directly without global synchronization, has a communication pattern that differs fundamentally from the centralized one studied in this paper. For fully decentralized asynchronous FL, although lightweight gossip-based schemes are more suitable than pre-calculated routing, resource competition may call for network slicing and multicast for FL of multiple models. Analyzing routing schemes and their complexity for decentralized asynchronous FL is a promising direction for future work.

\bibliographystyle{IEEEtran}
\bibliography{ref}

\end{document}